\documentclass[12 pt]{article}

\usepackage{arxiv}
\usepackage{lipsum}
\PassOptionsToPackage{linktocpage}{hyperref}
\usepackage{amsmath,amsfonts,amsthm,amscd,amssymb,mathrsfs,amssymb, multirow}
 \usepackage{natbib}  
\usepackage{booktabs} 
\usepackage[utf8]{inputenc} 
\usepackage[T1]{fontenc}    
\usepackage{stmaryrd}
\usepackage{hyperref}       
\usepackage{url}            
\usepackage{booktabs}       
\usepackage{nicefrac}       
\usepackage{microtype}      
\usepackage{lipsum}		
\usepackage{graphicx}
\usepackage{doi}
\usepackage{microtype}
\usepackage{booktabs} 
\usepackage{hyperref}
\usepackage{color, colortbl}
\definecolor{Gray}{gray}{0.9}
\usepackage{tabu, boldline}
\usepackage{graphicx}
\usepackage{algorithm} 
\usepackage{bbm}
\usepackage{algorithmic}
\usepackage{hyperref}
\usepackage{url}
\usepackage[utf8]{inputenc} 
\usepackage[T1]{fontenc}    
\usepackage{url}            
\usepackage{booktabs}       
\usepackage{nicefrac}       
\usepackage{microtype}      
\usepackage{graphicx}
\usepackage{upgreek}
\usepackage{bm}
\usepackage{enumerate}
\usepackage{enumitem}
\usepackage{xcolor}
\usepackage{cleveref}
\usepackage{wrapfig}

\usepackage{enumitem}
\usepackage{xr}
\usepackage{mathtools}
\usepackage[english]{babel}
\usepackage{tikz}
\usepackage{dirtytalk}
\usepackage{color, colortbl}
\usepackage{subcaption}

\title{\Large Provable Guarantees for Understanding Out-of-distribution Detection}

\author{%
  Peyman Morteza\\
  Department of Computer Sciences\\
  University of Wisconsin-Madison\\
  \texttt{peyman@cs.wisc.edu} \\
  \And
  Yixuan Li \\
  Department of Computer Sciences\\
  University of Wisconsin-Madison\\
  \texttt{sharonli@cs.wisc.edu} \\
  }
\date{}

\def\*#1{\mathbf{#1}}

\newtheorem{thm}{Theorem}
\newtheorem{lem}{Lemma}
\newtheorem{deff}{Definition}

\newtheorem{rmk}{Remark}
\newtheorem{prop}[thm]{Proposition}

\newtheorem{cor}{Corollary}

\numberwithin{figure}{section}
\definecolor{Gray}{gray}{0.9}
\definecolor{LightCyan}{rgb}{0.88,1,1}

\newcommand{\R}{\mathbb{R}}

\newcommand{\Norm}[1]{\left\Vert #1 \right\Vert}

\newcommand{\abs}[1]{\lvert#1 \rvert} 
\newcommand{\inn}[1]{\langle #1\rangle}

\newcommand{\B}{\mathcal{B}}

\DeclareMathOperator*{\innn}{in}
\DeclareMathOperator*{\out}{out}
\DeclareMathOperator{\featt}{feature}

 \hypersetup{
     colorlinks=true,
     linkcolor=blue,
     filecolor=magenta,      
     urlcolor=cyan,
 }
\def\*#1{\mathbf{#1}}

\begin{document}
\setcounter{tocdepth}{1}
\maketitle

\begin{abstract}
Out-of-distribution (OOD) detection is important for deploying machine learning models in the real world, where test data from shifted distributions can naturally arise. While a plethora of algorithmic approaches have recently emerged for OOD detection, a critical gap remains in theoretical understanding.
In this work, we develop an analytical framework that characterizes and unifies the theoretical understanding for OOD detection. Our analytical framework motivates a novel OOD detection method for neural networks, \emph{GEM}, which demonstrates both theoretical and empirical superiority. In particular, on CIFAR-100 as in-distribution data, our method outperforms a competitive baseline by {16.57}\% (FPR95). Lastly, we formally provide provable guarantees and comprehensive analysis of our method, underpinning how various properties of data distribution affect the performance of OOD detection\footnote{Code is available at:  \url{https://github.com/PeymanMorteza/GEM}}.

\end{abstract}

\section{Introduction}
\label{sec:intro}

When deploying machine learning models in the open world, it becomes increasingly critical to ensure the reliability---models are not only  accurate on their familiar data distribution, but also aware of unknown inputs outside the training data distribution. Out-of-distribution (OOD) samples can naturally arise from an irrelevant distribution whose label set has no intersection with training categories, and {therefore should not be predicted by the model}. This gives rise to the importance of OOD detection, which determines whether an input is in-distribution (ID) or OOD. 
 
 The main challenge in OOD detection stems from the fact that modern deep neural networks can easily produce overconfident predictions on OOD inputs~\citep{nguyen2015deep}. This phenomenon makes the separation between ID and OOD data a non-trivial task. OOD detection approaches commonly rely on an OOD scoring function that derives statistics from the pre-trained neural networks and performs OOD detection by exercising a threshold comparison. For example, ~\citep{hendrycks2016baseline} use the maximum softmax probability (MSP) and classifies inputs with smaller MSP scores as OOD data. While improved OOD scoring functions~\citep{liang2018enhancing, lee2018simple, liu2020energy, sun2021react} have emerged recently, {their inherent connections and theoretical understandings are largely lacking}. To the best of our knowledge, there is limited prior work providing provable guarantees for OOD detection methods from a rigorous mathematical point of view.

This paper takes an important step to bridge the gap by providing a unified framework that allows the research community to understand the theoretical connections among recent model-based OOD detection methods. Our framework further enables devising new methodology, theoretical and empirical insights on OOD detection. Our \textbf{key contributions} are three folds:
\begin{itemize}
    \item First, we provide an analytical framework that precisely characterizes and unifies the theoretical interpretation of several representative OOD scoring functions (Section~\ref{GDA}). We derive analytically an optimal form of OOD scoring function called \emph{GEM (Gaussian mixture based Energy Measurement)}, which is provably aligned with the true log-likelihood for capturing OOD uncertainty. In contrast, we show mathematically that prior scoring functions can be sub-optimal. 
    
    \item Second, our analytical framework motivates a new OOD detection method for deep neural networks (Section~\ref{sec:neural-networks}). By modeling the feature space as a class-conditional multivariate Gaussian distribution, we propose a \emph{GEM} score based on the Gaussian generative model.  Empirical evaluations demonstrate the competitive performance of the new scoring function. In particular, on CIFAR-100 as in-distribution data, \emph{GEM} outperforms \citep{liu2020energy} by {16.57}\% (FPR95). Our method is theoretically more rigorous than maximum Mahalanobis distance~\citep{lee2018simple} while achieving equally strong performance.  
    \item Lastly, our work provides both provable guarantees and empirical analysis to understand how various properties of data representation in feature and input space affect the performance of OOD detection (Section~\ref{sec:bound}). Previous OOD detection methods can be difficult to analyze due to the stochasticity in neural network optimization. Our framework offers key simplifications that allow us to (1)  isolate the effect of data representation from model optimization, and (2) flexibly modulate properties of data representation in feature and input space. Through both synthetic simulations and theoretical analysis, our study reveals important insights on how OOD detection performance changes with respect to data distributions. 
\end{itemize}
We end the introduction with an outline of this work. In Section \ref{GDA}, we first define the problem of study and set the notations that we need. Next, we analyze previous OOD detection methods under the Gaussian mixture assumption and introduce the GEM score. In Section \ref{sec:neural-networks}, we extend GEM to deep neural networks and perform experiments on common benchmarks. In Section \ref{sec:bound}, we provide rigorous guarantees for the performance of GEM, along with simulation verifications. We conclude our work in Section~\ref{sec:conclusion}, following an expansive literature review in Section~\ref{sec:related}.
\begin{figure*}[t]
\begin{center}
\includegraphics[scale=0.27]{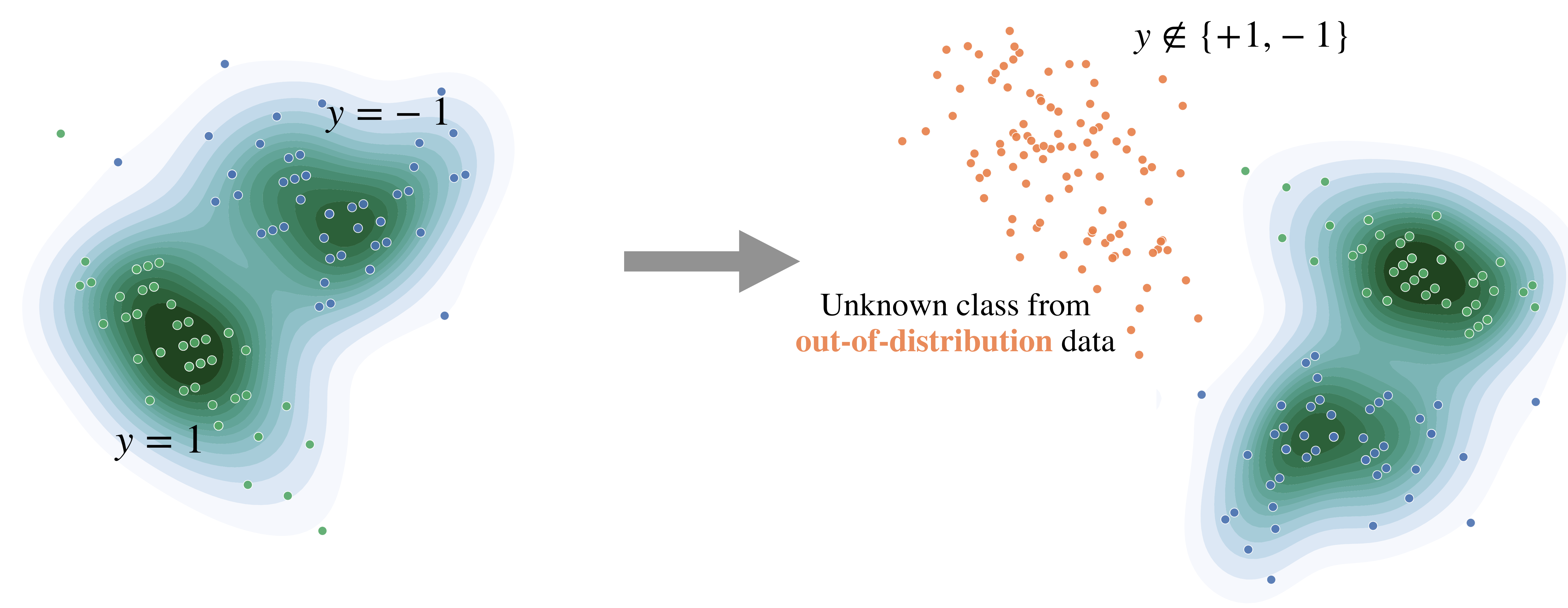}
\caption{\small \textbf{Left}: The in-distribution data $P^{\text{in}}_\mathcal{X}$ comprises of two classes $\mathcal{Y}=\{-1,+1\}$, indicated by green and blue dots respectively. 
\textbf{Right}: 
Out-of-distribution detection allows the
learner to express ignorance outside the support set of current known classes, and prevents the model from misclassifying OOD data (orange dots) into known classes (blue and green dots).}

\label{fig:splash}
\end{center}
\end{figure*}

\section{OOD Detection Under Gaussian Mixtures}
\label{GDA}
In this section, we mathematically describe representative OOD scoring functions under the Gaussian mixture data model. This allows us to contrast with the ideal OOD detector where the data density is explicit. We later apply the insight gained from this simple model to introduce a new score OOD detection for deep neural networks. 
\subsection{Preliminaries}
\label{sec:OOD}
We denote by $\mathcal{X}=\mathbb{R}^d$  the input space and $\mathcal{Y}=\{y_{1},..., y_{k}\}$  the label space. Let $P^{\innn}_{\mathcal{X},\mathcal{Y}}$ denote a probability distribution defined on $\mathcal{X}\times \mathcal{Y}$. Furthermore, let $P^{\innn}_{\mathcal{X}}$ and $P^{\innn}_{\mathcal{Y}}$ denote the marginal probability distribution on $\mathcal{X}$ and $\mathcal{Y}$ respectively. A classifier $f:\mathcal{X}\to \mathbb{R}^k$ learns to map a given  input $\*x \in \mathcal{X}$ to the output space . 

\paragraph{Problem Statement}
Given a classifier $f$ learned on training samples from in-distribution $P^{\innn}_{\mathcal{X},\mathcal{Y}}$, 
the goal is to design a binary function estimator, 
\begin{align*}
g: \mathcal{X} \to \{\innn,\out\},
\end{align*}
that classifies whether a test-time sample $\*x\in \mathcal{X}$ is generated from $P^{\innn}_{\mathcal{X}}$ or not.  Estimating OOD uncertainty is challenging due to the lack of knowledge on OOD data coming from  $P_{\mathcal{X}}^{\out}$. It is infeasible to explicitly train a binary classifier $g$.   A natural approach is to use  level set for OOD detection, based on the data density  $P^{\innn}_{\mathcal{X}}$. We define the {\it ideal classifier} for OOD detection as follows,
\[ 
	g^{\text{ideal}}_{\lambda}(\*x)=\begin{cases} 
      \text{in} & p^{\innn}_{\mathcal{X}}(\*x)\ge \lambda \\
      \text{out} & p^{\innn}_{\mathcal{X}}(\*x) < \lambda 
   \end{cases},
\]
where $p_{\mathcal{X}}^{\innn}$ is the density function of $P_{\mathcal{X}}^{\innn}$ and $\lambda$ is the threshold, which is chosen so that a high fraction (\emph{e.g.,} 95\%) of in-distribution data is correctly classified. 
For evaluation purpose, we define the error rate by,
\begin{align*}
	&\text{TPR}(g):=\mathbb{E}_{\*x \sim P_{\mathcal{X}}^{\innn}}(\mathbb{I}_{\{g(\*x)=\innn \}}),\\
	&\text{FPR}(g):=\mathbb{E}_{\*x \sim P_{\mathcal{X}}^{\out}}(\mathbb{I}_{\{g(\*x)=\innn\}}).
	\end{align*}
By convention, we assume in-distribution samples have positive labels. In practice, $P_{\mathcal{X}}^{\out}$ is often defined by a distribution that simulates unknowns encountered during deployment time, such as samples from an irrelevant distribution whose label set has no intersection with $\mathcal{Y}$ and {therefore should not be predicted by the model}.

\paragraph{In-distribution Data Model} We assume in-distribution data is drawn from a Gaussian mixture with equal priors and a tied covariance matrix $\Sigma$. The simplicity is desirable for us to precisely characterize various OOD detection methods and their optimality. We will further extend our analysis to neural networks in Section~\ref{sec:neural-networks}. Specifically,
\begin{align*}
	&\*x |  y_{i} \sim \mathcal{N}(\boldsymbol{\mu}_{i},\Sigma),\\
	&p^{\innn}_{\mathcal{Y}}(y_{i})=\frac{1}{k},
\end{align*}
where $\boldsymbol{\mu}_{i}\in \R^{d}$ is the mean of class $y_i \in \mathcal{Y}$ and $\Sigma \in \mathbb{R}^{d\times d}$ is the covariance matrix. The class-conditional density follows a Gaussian distribution,
\begin{align*}
p^{\innn}_{\mathcal{X}|\mathcal{Y}}(\*x|y_{i})=\frac{1}{\sqrt{(2\pi)^{d}\abs{\Sigma}}}\exp(-\frac{1}{2}(\*x-\boldsymbol{\mu}_{i})^{\top}\Sigma^{-1}(\*x-\boldsymbol{\mu}_{i})).
\end{align*}
Above implies the density function of $P_{\mathcal{X}}^{\innn}$ can be written as follows,
\begin{align*}
p_{\mathcal{X}}^{\innn}(\*x)& =\sum_{j=1}^{k}p^{\innn}_{\mathcal{X}|\mathcal{Y}}(\*x|y_{j})\cdot p^{\innn}_{\mathcal{Y}}(y_j)\\
& = \frac{1}{k\sqrt{(2\pi)^{d}\abs{\Sigma}}} \sum_{j=1}^{k} \exp(-\frac{1}{2}(\*x-\boldsymbol{\mu}_{j})^{\top}\Sigma^{-1}(\*x-\boldsymbol{\mu}_{j})),
\end{align*}
which is mixture of $k$ Gaussian distributions. 

\paragraph{Bayes Optimal Classifier} Under the Gaussian mixture model, the posterior probability of a Bayes optimal classifier for class $y_i \in \mathcal{Y}$ is given by,
\begin{align}
    p_{\mathcal{Y}|\mathcal{X}}(y_i|\*x) &= \frac{p_{\mathcal{Y}}(y_i)p_{\mathcal{X}|\mathcal{Y}}(\*x|y_i)}{\sum_{j=1}^{k} p_{\mathcal{Y}}(y_j)p_{\mathcal{X}|\mathcal{Y}}(\*x|y_j)}\\
    & = \frac{\exp(-\frac{1}{2}(\*x-\boldsymbol{\mu}_{i})^{\top}\Sigma^{-1}(\*x-\boldsymbol{\mu}_{i}))}{\sum_{j=1}^{k} \exp(-\frac{1}{2}(\*x-\boldsymbol{\mu}_{j})^{\top}\Sigma^{-1}(\*x-\boldsymbol{\mu}_{j}))}\\
    & = \frac{\exp f_i(\*x)}{\sum_{j=1}^{k} \exp f_j(\*x)} \label{eq:posterior},
\end{align}
where $f:\mathcal{X} \to \R^{k}$ is a function mapping to the logits. One can note that the above form of posterior distribution is equivalent to applying the softmax function on the logits $f(\*x)$, where,
\begin{align*}
f_{i}(\*x)=-\frac{1}{2}(\*x-\boldsymbol{\mu}_{i})^{\top}\Sigma^{-1}(\*x-\boldsymbol{\mu}_{i}),
\end{align*}
which is also known as the Mahalanobis distance~\citep{mahalanobis1936generalized}.

\subsection{OOD Scoring Functions and Their Optimality}
\label{sec:gmm_ood}
We now contrast several representative OOD scoring functions and also introduce our new scoring function GEM. Note that an ideal OOD detector should use a scoring function that is proportional to the data density. We focus on post hoc OOD detection methods, which have the advantages of being easy to use and general applicability without modifying the training procedure and objective. 
\paragraph{Prior: Maximum Softmax Score} 
\citeauthor{hendrycks2016baseline} propose using the maximum softmax score (MSP) for estimating OOD uncertainty,
\[ 
	g^{\text{MSP}}_{\lambda}(\*x)=\begin{cases} 
      \text{in} & \text{MSP}(f,\*x)\ge \lambda \\
      \text{out} & \text{MSP}(f,\*x) < \lambda 
  \end{cases}.
\]
The OOD scoring function is given by,
\begin{align*}
    \text{MSP}(f,\*x) &= \max_i p_{\mathcal{Y}|\mathcal{X}}(y_i|\*x) \\ & =\max_i \frac{1}{k\beta p_{\mathcal{X}}^{\innn}(\*x) }\exp(-\frac{1}{2}(\*x-\boldsymbol{\mu}_{i})^{\top}\Sigma^{-1}(\*x-\boldsymbol{\mu}_{i})) \\
    & \not\propto p^{\innn}_{\mathcal{X}}(\*x),
\end{align*}
where $\beta=\sqrt{(2\pi)^{d}\abs{\Sigma}}$. The above suggests that MSP is not aligned with the true data density, as illustrated in Figure~\ref{fig:function-2d}. For simplicity, we visualize the case when the input distribution is mixture of one-dimensional Gaussians, with two classes $\mathcal{Y}=\{+1,-1\}$. MSP can yield high score 1, and misclassify data points in low-likelihood regions such as $x>4$ or $x<-4$ (highlighted in red). Also, depending on threshold value $\lambda$, MSP may misclassify samples from neighbourhood around the origin (highlighted in gray).

\paragraph{Prior: Maximum Mahalanobis Distance} \citeauthor{lee2018simple} propose using the maximum Mahalanobis distance \emph{w.r.t} the closest class centroid for OOD detection. Specifically, the score is defined as:
\begin{align*}
    M(f, \*x) & = \max_i -(\*x-\boldsymbol{\mu}_{i})^{\top}\Sigma^{-1}(\*x-\boldsymbol{\mu}_{i})\\
    & \not\propto p^{\innn}_{\mathcal{X}}(\*x),
\end{align*}
which is equivalent to the maximum Mahalanobis distance. 
The corresponding OOD classifiers based on Mahalanobis score is,
\[ 
	g^{\text{Mahalanobis}}_{\lambda}(\*x)=\begin{cases} 
      \text{in} & M(f, \*x)\ge \lambda \\
      \text{out} & M(f, \*x) < \lambda 
  \end{cases}.
\]
The above suggests that Mahalanobis distance is not proportional to the true data density either, hence sub-optimal. 
\begin{figure}[t]
\begin{center}
\includegraphics[width=125 mm]{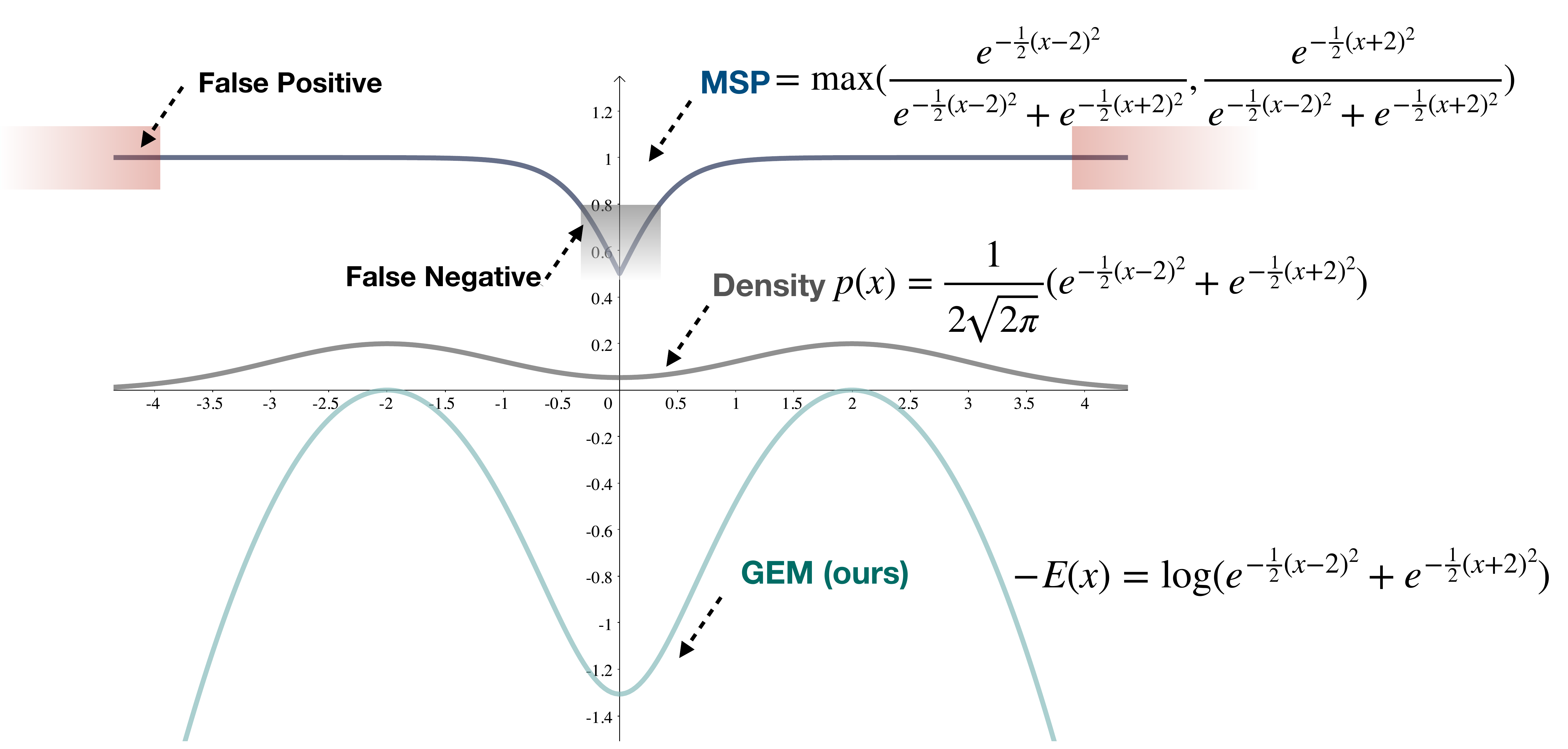}
\caption{\small Illustration of data density (\textbf{middle}, gray), maximum softmax score (\textbf{top}, black) and GEM score (\textbf{bottom}, blue) when $x\in \mathbb{R}$. The data distribution $p(x)$ is a mixture of two Gaussians with mean $
\mu_1=2$ and $\mu_2=-2$ respectively. The variance $\sigma_1=\sigma_2=1$. Under thresholding, MSP can not distinguish samples $x>4$ or $x<-4$ which have low-likelihood of being in-distribution. In contrast, GEM score (ours) is aligned with the true data density, and better captures OOD uncertainty.}
\vspace{-0.4cm}
\label{fig:function-2d}
\end{center}
\end{figure}
\paragraph{Prior: Energy Score} Given a function transformation $f: \mathcal{X} \to \mathbb{R}^k$, ~\citeauthor{liu2020energy}  propose using the free energy score for OOD detection. The free energy is defined to be the \texttt{-logsumexp} of logit outputs,
\begin{align}
\label{eq:energy}
\text{E}(f,\*x)& =-\log \sum_{j=1}^{k}\exp(f_{j}(\*x)),
\end{align}
where $f(\*x)=(f_{1}(\*x),...,f_{k}(\*x))^{\top} \in \R^{k}$. 
We provide a simple and concrete example to show there exists maximum likelihood solution with the same posterior probability as in Equation~\ref{eq:posterior}, but the resulting energy score is not aligned with the data density:
\begin{align*}
    p_{\mathcal{Y}|\mathcal{X}}(y_i|\*x) 
    & = \frac{\exp(-\frac{1}{2}(\*x-\boldsymbol{\mu}_{i})^{\top}\Sigma^{-1}(\*x-\boldsymbol{\mu}_{i}))}{\sum_{j=1}^{k} \exp(-\frac{1}{2}(\*x-\boldsymbol{\mu}_{j})^{\top}\Sigma^{-1}(\*x-\boldsymbol{\mu}_{j}))} \\
    & = \frac{\exp(\boldsymbol{\mu}_i^{\top}\Sigma^{-1}\*x-\frac{1}{2}\boldsymbol{\mu}_{i}^{\top}\Sigma^{-1}\boldsymbol{\mu}_{i})}{\sum_{j=1}^{k} \exp(\boldsymbol{\mu}_j^{\top}\Sigma^{-1}\*x-\frac{1}{2}\boldsymbol{\mu}_{j}^{\top}\Sigma^{-1}\boldsymbol{\mu}_{j})} \\
    &= \frac{\exp f'_i(\*x)}{\sum_{j=1}^{k} \exp f'_j(\*x)}, 
\end{align*} 
where $f'(\*x):=(f'_{1}(\*x),...,f'_{k}(\*x))^{\top}\in \R^{k}$ can be viewed as a single layer network's output with (row) weights $\boldsymbol{\mu}_i^{\top}\Sigma^{-1}$, for $1\le i\le k$, and biases $-\frac{1}{2}\boldsymbol{\mu}_{i}^{\top}\Sigma^{-1}\boldsymbol{\mu}_{i}$, for $1\le i\le k$, and the corresponding energy $-\log \sum_j \exp f'_j(\*x)$ is not aligned with the log-likelihood, hence not always optimal.

\paragraph{New: GEM Score} We now introduce a new scoring function, Gaussian mixture based energy measurement (dubbed \emph{GEM}). The GEM score can be written as,
\begin{align*}
\text{GEM}(f,\*x)  = - E(f,\*x) =  &  \log \sum_{j=1}^k \exp(-\frac{1}{2}(\*x-\boldsymbol{\mu}_{j})^{\top}\Sigma^{-1}(\*x-\boldsymbol{\mu}_{j}))\\
& \propto \log p_{\mathcal{X}}^{\innn}(\*x), 
\end{align*}
 which suggests that the \emph{GEM score} is proportional (by ignoring a constant term) to the log-likelihood of the in-distribution data. Note that we flip the sign of free energy to align with the convention that larger GEM score indicates more ID-ness and vice versa.  
The key difference here is that the GEM score is a \emph{special case} of negative free energy, where each $f_j(\*x)$ in Equation~\ref{eq:energy} takes on the form of Mahalanobis distance instead of directly using the logit outputs,
\begin{align*}
f_{j}(\*x)=-\frac{1}{2}(\*x-\boldsymbol{\mu}_{j})^{\top}\Sigma^{-1}(\*x-\boldsymbol{\mu}_{j}).
\end{align*}
In Figure~\ref{fig:function-2d}, we show the alignment between the GEM (light green) and true data density function (gray), in a simplified case with $x\in \mathbb{R}$, $k=2$ and $\mu_1=2,\mu_2=-2$. The corresponding OOD classifiers based on energy score is,

\[ 
	g^{\text{GEM}}_{\lambda}(\*x)=\begin{cases} 
      \text{in} & \text{GEM}(f,\*x)\ge \lambda \\
      \text{out} & \text{GEM}(f,\*x) < \lambda 
  \end{cases}.
\]
This leads to the following lemma that shows the optimality of the GEM estimator. 
\begin{lem}
\label{lem:simple_gmm}
In the case of Gaussian conditional with equal priors, the GEM based OOD estimator performs similarly to the ideal classifier defined in Section \ref{sec:OOD}. More specifically,
\begin{align*}
g^{ideal}_{\lambda}=g^{\text{GEM}}_{\log(k\beta\cdot\lambda)},
\end{align*}
where $\beta=\sqrt{(2\pi)^{d}\abs{\Sigma}}$ and both the ideal classifier and our method are aligned with $P_{\mathcal{X}}^{\innn}$ by definition. 
\end{lem}
\begin{rmk}
\label{rmk:weighted}
\normalfont
Note that the equal prior case is considered to convey the main idea in simplest possible form. To make it more general, a weighted version of  GEM can be used to achieve the optimality for the non-equal prior case. More precisely, let  $w_{i}=p^{\innn}_{\mathcal{Y}}(y_{i})$, then we have,
\begin{align*}
p_{\mathcal{X}}^{\innn}(\*x) &=\sum_{j=1}^{k}w_{j}p^{\innn}_{\mathcal{X}|\mathcal{Y}}(\*x|y_{j})\\
& \propto \sum_{j=1}^{k} w_{j}\exp(-\frac{1}{2}(\*x-\boldsymbol{\mu}_{j})^{\top}\Sigma^{-1}(\*x-\boldsymbol{\mu}_{j})).
\end{align*}
Now if we define the {\it weighted GEM} by,
\begin{align*}
    \text{GEM}^{w}(f,\*x) & := \log \sum_{j=1}^k w_{j}\exp(-\frac{1}{2}(\*x-\boldsymbol{\mu}_{j})^{\top}\Sigma^{-1}(\*x-\boldsymbol{\mu}_{j})),
\end{align*}
then arguing similar to Lemma \ref{lem:simple_gmm} implies that weighted GEM would be aligned with the ideal classifier in the non-equal prior case.
\end{rmk}

\begin{wrapfigure}{r}{0.3\textwidth}
\begin{center}
\vspace{-0.6cm}
\includegraphics[width=0.3\textwidth]{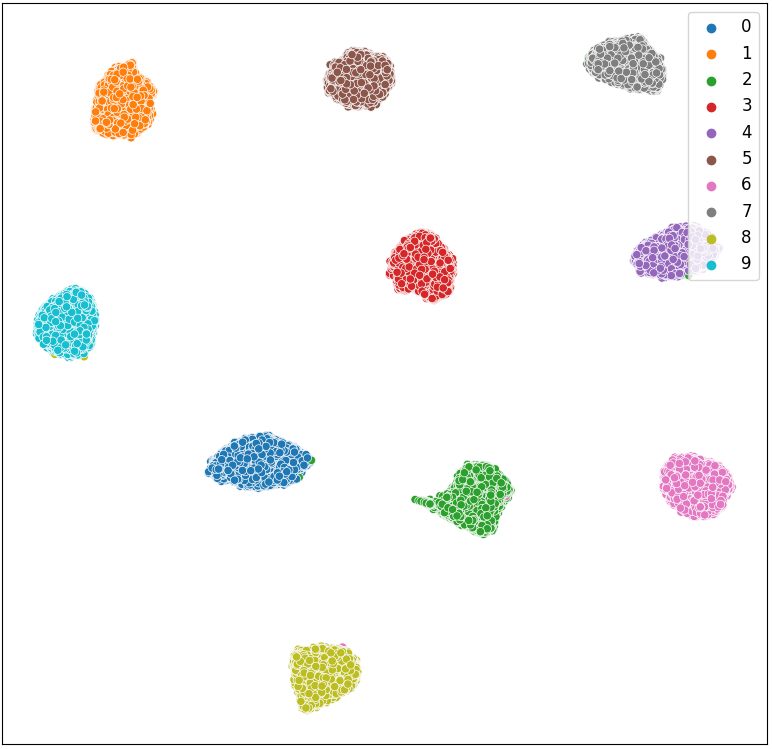}
\caption{\small UMAP visualization of embeddings for CIFAR-10 model. 
}
\label{fig:umap} 
\vspace{-0.6cm}
\end{center}
\end{wrapfigure}

\section{OOD Detection for Deep Neural Networks}
\label{sec:neural-networks}
In this section, we extend our analysis and method to deep neural networks. To start, let $h(\*x;\theta) \in \mathbb{R}^m$ be the feature vector of the input $\*x$, extracted from the penultimate layer of a neural net parameterized by $\theta$. 
We assume that a class-conditional distribution in the feature space follows the multivariate Gaussian distribution. Such an assumption has been empirically validated in~\citep{lee2018simple}; also see visualizations in Figure~\ref{fig:umap}. Specifically, a $k$ class-conditional Gaussian distribution with a tied covariance is defined as, 
\begin{align*}
    h(\*x;\theta)|y_i \sim \mathcal{N}(\boldsymbol{u}_{i}, \bar{\Sigma}),
\end{align*}
where $\boldsymbol{u}_{i}\in \R^{m}$ is the mean of class $y_i$ and $\bar{\Sigma} \in \mathbb{R}^{m\times m}$ is the covariance matrix. To estimate the parameters of the generative model from the pre-trained neural classifier,
one can compute the empirical class mean and covariance given training samples $\{(\*x_1,\bar{y}_1), (\*x_2,\bar{y}_2),...,(\*x_N,\bar{y}_N)\}$,
\begin{align*}
    \boldsymbol{\hat u}_i &= \frac{1}{N_i} \sum_{j:\bar{y}_j=y_{i}} h(\*x_j;\theta), \\
    \hat\Sigma & = \frac{1}{N} \sum_{i=1}^{k} \sum_{j:\bar{y}_j=y_{i}} (h(\*x_j;\theta)-\boldsymbol{\hat u}_i)(h(\*x_j;\theta)-\boldsymbol{\hat u}_i)^T,
\end{align*}
where $N_i$ is the number of training samples with label $y_{i}\in \mathcal{Y}$. We can define  the ideal classifier with respect to {\it feature space} to be,
\begin{align}
\label{ideal-GMM}
	g^{\text{ideal}}_{\lambda}(\*x)=\begin{cases} 
      \text{in} & p^{\featt}(\*x)\ge \lambda \\
      \text{out} & p^{\featt}(\*x) < \lambda,
  \end{cases},
\end{align}
where $p^{\featt}$ denotes the density function of the posterior distribution on the feature space induced by $h(\*x,\theta)$. 

\paragraph{GEM for Neural Networks} Similar to our definition in Section~\ref{GDA}, \emph{GEM} for neural networks can be defined as
\begin{align*}
\text{GEM}(\*x;\theta)& =\log \sum_{j=1}^{k}\exp(f_{j}(\*x;\theta)),
\end{align*}
where $f_j(\*x;\theta)= -\frac{1}{2}(h(\*x;\theta)-\boldsymbol{u}_{j})^{\top}{\bar{\Sigma}}^{-1}(h(\*x;\theta)-\boldsymbol{u}_{j})$. We can empirically estimate each $f_j(\*x;\theta)$ by, 
\begin{align*}
\hat f_j(\*x;\theta)= -\frac{1}{2}(h(\*x;\theta)-\boldsymbol{\hat u}_{j})^{\top}{\hat{\Sigma}}^{-1}(h(\*x;\theta)-\boldsymbol{\hat u}_{j}).
\end{align*}
It follows from an analogue of Lemma \ref{lem:simple_gmm} that $g^{\text{GEM}}_{\lambda}$, computed from feature space, performs similarly to the ideal classifier that we defined by Equation \ref{ideal-GMM}.
\begin{lem}
\label{lem:feature_gmm}
The performance of GEM based detection is same as the ideal classifier (with respect to the feature space) defined by Equation \ref{ideal-GMM} :
\begin{align*}
g^{ideal}_{\lambda}=g^{\text{GEM}}_{\log(k\bar{\beta}\cdot\lambda)},
\end{align*}
where $\bar{\beta}=\sqrt{(2\pi)^{m}\abs{\bar{\Sigma}}}$.
\end{lem}
We also note that Lemma \ref{lem:feature_gmm} can be extended to non-equal prior case by arguing similar to Remark \ref{rmk:weighted}.
\begin{table*}[t!]
\centering
\scalebox{0.95}{
                \begin{tabular}{llrrrr}
                        \toprule
     \multirow{3}{0.08\linewidth}{In-distribution} &  \multirow{3}{0.06\linewidth}{\textbf{Method}}  &\textbf{FPR95}  & \textbf{AUROC}  & \textbf{AUPR} & \textbf{In-dist}  \\
                        & &  &  & $\textbf{}$ & $\textbf{Test Error}$  \\
                        & & $\downarrow$  & $\uparrow$ & $\uparrow$ & $\downarrow$ \\ \midrule
           \multirow{6}{0.1\linewidth}{{\textbf{ CIFAR-10} }}
                        & Softmax score \citep{hendrycks2016baseline}
                        & 51.04 & 90.90 & 97.92 & 5.16\\
                        & ODIN \citep{liang2018enhancing}
                        & 35.71 & 91.09 & 97.62 & 5.16 \\
                        & Mahalanobis \citep{lee2018simple}
                        & 36.96 & 93.24 & 98.47 & 5.16 \\
                         &  Energy score~\citep{liu2020energy}
                        & 33.01 & 91.88 & 97.83 & 5.16\\
                      &  GEM (ours) 
                     & 37.21 & 93.23 & 98.47 & 5.16\\
                        \midrule
                 \multirow{6}{0.12\linewidth}{\textbf{CIFAR-100} }
                        & Softmax score \citep{hendrycks2016baseline}
                        & 80.41 & 75.53 & 93.93 & 24.04\\
                        & ODIN \citep{liang2018enhancing}
                        & 74.64 & 77.43 & 94.23 & 24.04\\
                        & Mahalanobis \citep{lee2018simple}
                        & 57.01 & 82.70 & 95.68 & 24.04\\
                         &    Energy score~\citep{liu2020energy}
                        & 73.60 & 79.56 & 94.87 & 24.04\\
                        &  GEM (ours)    
                        & 57.03 & 82.67 & 95.66& 24.04\\
                        \bottomrule
                \end{tabular}}
        \vspace{-0.2cm}
        \caption[]{\small \textbf{Main Results}. Comparison with  competitive \emph{post hoc} OOD detection methods. $\uparrow$ indicates larger values are better, and $\downarrow$ indicates smaller values are better. All values are percentages. Results for OOD detection are averaged over the six OOD test datasets  described in section \ref{sec:experiment-nn}. Numbers for individual OOD test datasets are in the Appendix. The reported results for benchmarks other than GEM are courtesy of \citep{liu2020energy}.}
        \label{tab:main-results}
        \vspace{-0.2cm}
\end{table*}

\subsection{Experimental Results}
\label{sec:experiment-nn}
\paragraph{Setup} 

We use  CIFAR-10 and CIFAR-100~\citep{krizhevsky2009learning} datasets as in-distribution data. We use the standard split, and train with WideResNet architecture~\citep{zagoruyko2016wide} with depth 40.  For the OOD test dataset, we use the following six datasets: \texttt{Textures}~\citep{cimpoi2014describing}, \texttt{SVHN}~\citep{netzer2011reading}, \texttt{Places365}~\citep{zhou2017places}, \texttt{LSUN-Crop}~\citep{yu2015lsun}, \texttt{LSUN-Resize}~\citep{yu2015lsun}, and \texttt{iSUN}~\citep{xu2015turkergaze}. 
We report standard metrics including FPR95 (false positive rate of OOD examples when the true positive rate of in-distribution examples is at 95\%), AUROC, and AUPR.

\paragraph{GEM is both empirically competitive and theoretically grounded.}
Table \ref{tab:main-results} compares the performance of the GEM method with common OOD detection methods. For fairness, all methods derive OOD scoring functions post hoc from the same pre-trained model. For example, on CIFAR-100 as in-distribution data, GEM outperforms the energy score~\citep{liu2020energy} by {16.57}\% (FPR95). Compared to ~\citep{lee2018simple}, our method is more theoretically grounded than taking the maximum Mahalanobis distance. We note that the similar empirical performance is primarily due to \texttt{log-sum-exp} being a smooth approximation of maximum Mahalanobis distance in the feature space (more details in Remark \ref{rmk:mahal} below). Therefore, our method overall achieves both strong empirical performance and theoretical soundness---bridging a critical gap under unified understandings.

\subsection{Remarks}

\begin{rmk}[Significance \emph{w.r.t} Mahalanobis] 
\label{rmk:mahal}
\normalfont
The main difference \emph{w.r.t}~\citep{lee2018simple} is that we are taking the \texttt{log-sum-exp} over Mahalanobis distances  $M_i$, instead of taking the \texttt{maximum} Mahalanobis distance. This was  motivated by our theoretical analysis in previous Section where taking \texttt{log-sum-exp} would be aligned with likelihood (\emph{w.r.t} feature space), whereas \texttt{max} is not exact in theory. In other words, {we bring theoretical rigor to an empirically competitive method}. Mathematically, $\log \sum_i^k \exp(M_i)\approx \max_i M_i$ with the following bound:
$
\max_i M_i \le \log \sum_i^k \exp(M_i) \le \max_i M_i + \log (k)
$. Therefore, our method overall achieves equally strong empirical performance yet with theoretical soundness and guarantees (see formal analysis in Section~\ref{sec:bound}). 
\end{rmk}
\begin{lem}
\label{lem:mahal_gmm}
In the case of Gaussian conditional with equal priors in the feature space, the Mahalanobis-based OOD estimator is not aligned with the density of in-distribution data in the feature space and it is not equivalent to the ideal classifier defined by Equation \ref{ideal-GMM}.
\end{lem}

\begin{rmk}[Significance \emph{w.r.t} Energy Score]
\label{rmk:energy}
\normalfont
The energy score in \citep{liu2020energy} was derived directly from the logit outputs, rather than a Gaussian generative model as in ours. As a result, the original energy score might not always correspond to the Bayes optimal logit to ensure alignment \emph{w.r.t} likelihood (we showed this by an explicit example in Section \ref{GDA}). Instead, our analytical
framework and method provide strong provable guarantees (c.f. Section~\ref{sec:bound}) and enable precise understanding by disentangling the effects of various factors (c.f. Section~\ref{sec:bound}), both of which were not presented in \citep{liu2020energy}. Moreover, we show empirically that GEM achieves strong empirical performance, outperforming energy score by a significant margin ({16.57}\% in FPR95 on CIFAR-100, see Table~\ref{tab:main-results}).
\end{rmk}

\section{Provable Guarantees for GEM}
\label{sec:bound}
The main goal of this section is to provide rigorous guarantees and understandings for our method GEM. This is important but often missing in previous literature on OOD detection. 

Let $P_{\mathcal{X}}^{\innn}$ be a mixture of Gaussians (similar to Section \ref{GDA}) and assume $P^{\out}_{\mathcal{X}}=\mathcal{N}(\boldsymbol{\mu}_{\out},\Sigma)$. We can think of $\mathcal{X}$ as either the feature space or input space of a deep neural net. We work with the re-scaled version of the GEM score (by omitting the $\log$ operator), which does not change the formal guarantees. 
\begin{align*}
	ES(\*x)=\sum_{i=1}^{k}ES_{i}(\*x),
\end{align*}
where,
\begin{align*}
	ES_{i}(\*x)=\exp(-\frac{1}{2}(\*x-\boldsymbol{\mu}_{i})^{\top}\Sigma^{-1}(\*x-\boldsymbol{\mu}_{i})).
	\end{align*}
Next, we consider the following quantity,
\begin{align*}
\label{ast}
D:=\mathbb{E}_{\*x \sim P^{\innn}_{\mathcal{X}}}	(ES(\*x))-\mathbb{E}_{\*x \sim P^{\out}_{\mathcal{X}}}	(ES(\*x)).
\end{align*}
Intuitively, we can think of $D$ as a measure of how well GEM distinguishes ID samples from OOD samples. For example, when $\boldsymbol{\mu}_{\out}$ is far away from $\boldsymbol{\mu}_{i}$ then we expect $\mathbb{E}_{\*x \sim P^{\out}_{\mathcal{X}}}(ES(\*x))$ to be small (i.e., $D$ is large), and we expect that the our OOD estimator performs better compared to the case when $\boldsymbol{\mu}_{\out}$ is close to $\boldsymbol{\mu}_{i}$ (i.e., $D$ is small). We make this intuition precise by bounding $D$ in terms of Mahalanobis distance between $\boldsymbol{\mu}_{\out}$ and $\boldsymbol{\mu}_{i}$. First, we recall the following definition and set some notations,
\begin{deff}
\label{def:mahal}
 For $\*u,\*v \in \R^{d}$, the Mahalanobis distance, with respect to $\Sigma$, is defined by,
\begin{align*}
d_{M}(\*u,\*v):=\sqrt{(\*u-\*v)^{\top}\Sigma^{-1}(\*u-\*v)},
\end{align*}
and for $r>0$ the open ball with center $\*u$ and radius $r$ is defined by,
\begin{align*}
B_{r}(\*u):=\{\*x \in \R^{d}|d_{M}(\*x,\*u)< r\}.
\end{align*}
\end{deff}
Next, we can state the following theorem. 
\begin{thm}
\label{paper:main_thm}
We have the following bounds,	
\begin{itemize}
    \item 
    
	$\mathbb{E}_{\*x \sim P^{\out}_{\mathcal{X}}}	(ES(\*x))\le \sum_{i=1}^{k}\Big{(}\big{(}1-P_{\mathcal{X}}^{\out}(B_{\alpha_{i}}(\boldsymbol{\mu}_{\out}))\big{)}+\exp(-\frac{1}{2}\alpha_{i}^{2})\Big{)}$,
	\item 
	$\mathbb{E}_{\*x \sim P^{\innn}_{\mathcal{X}}}	(ES(\*x))-\mathbb{E}_{\*x \sim P^{\out}_{\mathcal{X}}}	(ES(\*x))\le \sum_{i=1}^{k}\alpha_{i}$,
\end{itemize}
where, for $1\le i\le k$, $\alpha_{i}:=\frac{1}{2}d_{M}(\boldsymbol{\mu}_{i},\boldsymbol{\mu}_{\out})$.
\end{thm}
We emphasize that in Theorem \ref{paper:main_thm} $\boldsymbol{\mu}_{i}$ and $\boldsymbol{\mu}_{\out}$ can have {\it arbitrary configurations}. We refer the reader to the Appendix for the proof of Theorem \ref{paper:main_thm} and detailed discussions on other variants. 

\paragraph{Performance with respect to the distance between ID and OOD data}
The next corollary explains how Theorem \ref{paper:main_thm} can quantify that the performance of GEM-based OOD detector increases as the distance between ID and OOD data increases. 
\begin{cor}
\label{main-cor-1}
For $1\le i\le k$, set $\alpha=d_{M}(\boldsymbol{\mu}_{out},\boldsymbol{\mu}_{i})$. We have the following from the first bound in Theorem \ref{paper:main_thm},
\begin{align*}
	& \mathbb{E}_{\*x \sim P^{\out}_{\mathcal{X}}}	(ES(\*x))\le k\Big{(}\big{(}1-P_{\mathcal{X}}^{\out}(B_{\alpha}(\boldsymbol{\mu}_{\out}))\big{)}+\exp(-\frac{1}{2}\alpha^{2})\Big{)}.
\end{align*}
Now as $\alpha \to \infty$ the right hand side in the above approaches to $0$. This indicates that the performance of our method improves as $\alpha \to \infty$. On the other hand, using the second bound in the Theorem \ref{paper:main_thm}, we have,
\begin{align*}
	 \mathbb{E}_{x \sim P^{\innn}_{\mathcal{X}}}	(ES(\*x))-\mathbb{E}_{\*x \sim P^{\out}_{\mathcal{X}}}	(ES(\*x))\le k\alpha,
\end{align*}
and it follows that as $\alpha \to 0$ the energy difference between in-distribution and out-of-distribution data converges to $0$. In other words, the performance  decreases as $\alpha$ approaches to $0$. We will further justify our theory in simulation study (next subsection). 
 \end{cor}

\paragraph{Performance in high dimensions}
We now show that the performance of GEM decreases as dimension of feature space increases. This is due to {\it curse of dimensionality} which we next explain. First, for simplicity assume that $\boldsymbol{\mu}_{out}=0$ and for all $1\le i\le k$, $\alpha=d_{M}(\boldsymbol{\mu}_{out},\boldsymbol{\mu}_{i})$. Consider a multi-dimensional gaussian $\mathcal{N}(0,\mathbf{I}_{d})$. As $d$ increases the high-probability region under this gaussian distribution will concentrate away from the origin. More precisely, 
\begin{align*}
\*x \sim \mathcal{N}(0,\mathbf{I}_{d}) \implies   \Norm{\*x}_{2}^{2} \sim \chi^{2}_{d} \implies \mathbb{E}(\Norm{\*x}_{2}^{2})=d.
\end{align*}
Therefore, the out-of-distribution samples will have a larger distance (on average) to the origin as dimension increases and it follows that the OOD detector may misclassify these OOD samples as in-distribution.

We next conduct several simulation studies to systematically verify our provable guarantees. 
\subsection{Simulation Studies and Further Analysis}
\label{sec:gmm-simulation}

What properties of the data representation make OOD uncertainty challenging? In this subsection, we construct a synthetic data representation that allows us to flexibly modulate different properties of the data representation including:
\begin{enumerate}[label=(\roman*)]
\item distance between ID and OOD data, 
\item feature or input dimension,
\item number of classes.
\end{enumerate}
 We simulate and probe how these factors affect OOD uncertainty estimation. The simulation also serves as a verification of our theoretical guarantees.
\paragraph{Feature representation setup} The in-distribution representation on the feature space (or input space) comprises a mixture of $k$ class-conditional Gaussian. To replicate common empirical benchmarks such as CIFAR-10 and CIFAR-100~\citep{krizhevsky2009learning}, we explore both $k=10$ and $k=100$ by default. Unless otherwise specified, we set the feature (or input) dimension $d=512$. We fix the total number of in-distribution samples $N=20,000$. The tied covariance matrix is diagonal with magnitude $\sigma^2$, i.e., $\Sigma=\sigma^2\mathbf{I}_d$.

We assume the data in the feature space (or input space) $\*x \in \mathbb{R}^d$ is sampled from the following class-conditional Gaussian,
\begin{align*}
    \*x_\text{in}~|~y_i \sim \mathcal{N}(\boldsymbol{\mu}_i, \sigma^2 \mathbf{I}_d),
\end{align*}
where $\boldsymbol{\mu}_i$ is the mean for in-distribution classes $i\in \{1,2,...,k\}$. We consider different configurations of $\boldsymbol{\mu}_{i}$, $1\le i \le k$ representing means of each $k$ in-distribution classes.
Specifically, the mean $\boldsymbol{\mu}_i$ corresponding to  $i$-th class is a unit vector $\boldsymbol{v}_i$, scaled by a distance parameter $r>0$.  In particular, $\boldsymbol{\mu}_{i}=r\cdot \boldsymbol{\nu}_{i}$, where $\Norm{\boldsymbol{\nu}_i}_2=1$. $\boldsymbol{\nu}_i$ is a sparse vector with $s=\lfloor d/k\rfloor$  non-zero entries, with equal values in the position from $(i-1)\cdot s$ up to $i\cdot s $ and $0$ elsewhere. It follows that for $i,j \in \{1,...,k\}$ and $i \neq j$, 

\begin{equation}
    \begin{aligned}
        &\inn{\boldsymbol{\nu}_{i},\boldsymbol{\nu}_{j}}=0,\\
    &\Norm{\boldsymbol{\nu}_{i}-\boldsymbol{\nu}_{j}}_{2}=\sqrt{2}.
\label{equation:sim-properties}
\end{aligned}
\end{equation}
Furthermore, we assume that the out-of-distribution data representation is centered at the origin, with $\boldsymbol{\mu_\text{out}}=\mathbf{0} \in \mathbb{R}^d$,
\begin{align*}
    \*x_\text{out} \sim \mathcal{N}(\mathbf{0}, \sigma^2 \mathbf{I}_d).
\end{align*}
Note that the above configuration is considered for simplicity and similar simulation results holds when we translate $\boldsymbol{\mu}_{i}$ and $\boldsymbol{\mu}_{\out}$ with a constant vector or by applying an orthogonal transformation.

\paragraph{Rationale of the synthetic data} Compared to estimating GEM scores from real datasets using parameterized models (such as neural networks), 
these synthetic simulations offer two key simplifications. 
First, viewing the setting on feature space, we can \emph{flexibly modulate} key properties of datasets such as the number of classes and distance between induced ID and OOD representation in the feature space. In contrast, in real datasets, these properties are usually predetermined. Second, viewing the setting on input space, the function mapping $f(\*x)$ is completely deterministic and optimal, provided with known parameters $\{\boldsymbol{\mu}_1, \boldsymbol{\mu}_2,...,\boldsymbol{\mu}_k\}$ and $\Sigma$. This allows us to isolate the effect of data distribution from model optimization. In contrast, estimating $f(\*x)$ using complex models such as neural networks might have inductive bias, and depend on the optimization algorithm chosen.

\paragraph{Performance  with respect to the number of classes}
We now show that the performance of our method decreases as the number of classes increases. To explain this, we compute $D$ in terms of $k$ to see how they are related. First, we need the following definition,
\begin{deff}
Let $\boldsymbol{\mu} ,\boldsymbol{\nu} \in \R^{d}$ with $\gamma=\Norm{\boldsymbol{\mu}-\boldsymbol{\nu}}_{2}$. Let $P \sim \mathcal{N}(\boldsymbol{\mu},\mathbf{I}_{d})$. Define,
\begin{align*}
    A_{\gamma}:=\mathbb{E}_{\*x \sim P}	(\exp(-\frac{1}{2}\Norm{\*x-\boldsymbol{\nu}}_{2}^{2})).
\end{align*}
\end{deff}
\begin{rmk}
\label{rotation}
Notice that, since standard Gaussian distribution is rotationally invariant, $A_{\gamma}$ only depends on the distance between $\boldsymbol{\mu}$ and $\boldsymbol{\nu}$ (i.e. $\gamma$). Also it is easy to see that $A_{\gamma}$ decreases as $\gamma$ increases.
\end{rmk}
\begin{prop}
\label{num-class}
We have the following,
\begin{align*}
 D=A_{0}-A_{r}+(k-1)(A_{\sqrt{2}\cdot r}-A_{r}).
\end{align*}
\end{prop}
We refer the reader to the Appendix for the proof of Proposition \ref{num-class}. The next Corollary explains how the performance of our method decreases by increasing the number of classes.
\begin{cor}
Since $\sqrt{2}\cdot r>r$, it follows from Remark \ref{rotation} that $A_{\sqrt{2}\cdot r} <A_{r}$. This means that the last term in the following is negative,
\begin{align*}
   D=A_{0}-A_{r}+(k-1)(A_{\sqrt{2}\cdot r}-A_{r}).
\end{align*}
In other words, as $k$ increases $D$ becomes smaller which indicates that the performance of the GEM method decreases.
\end{cor}
\subsection{Simulation Results}
In this subsection, we report simulation results that confirm our theoretical guarantees presented above. 
\paragraph{Effect of distance between ID and OOD} Figure~\ref{fig:simulation} (left) shows how the False Positive Rate (at 95\% TPR) changes with the distance between ID and OOD features. The $\sigma$ is set to be $1$ and the distance is modulated by adjusting the magnitude parameter $r$, where a larger $r$ results in a larger distance. For both $k=10$ and $k=100$, the FPR decreases as the distance increases, which matches our intuition that more drastic distribution shifts are easier to be detected. Under the same distance, we observe a relatively higher FPR for data with more classes ($k=100$). The performance gap diminishes as the distance becomes very large. 

\begin{figure}[t]
\begin{center}

\includegraphics[width=51mm]{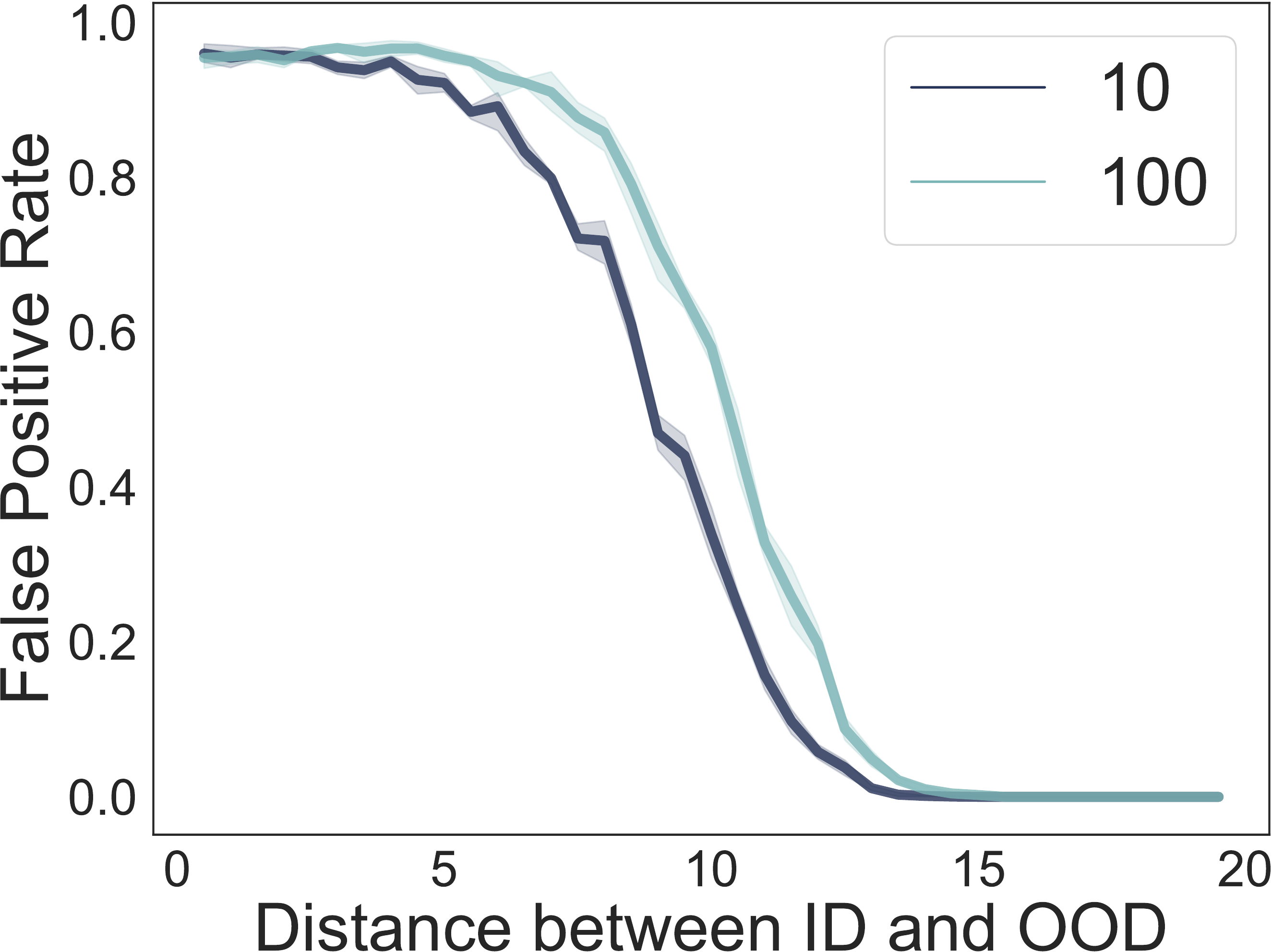}
\includegraphics[width=51mm]{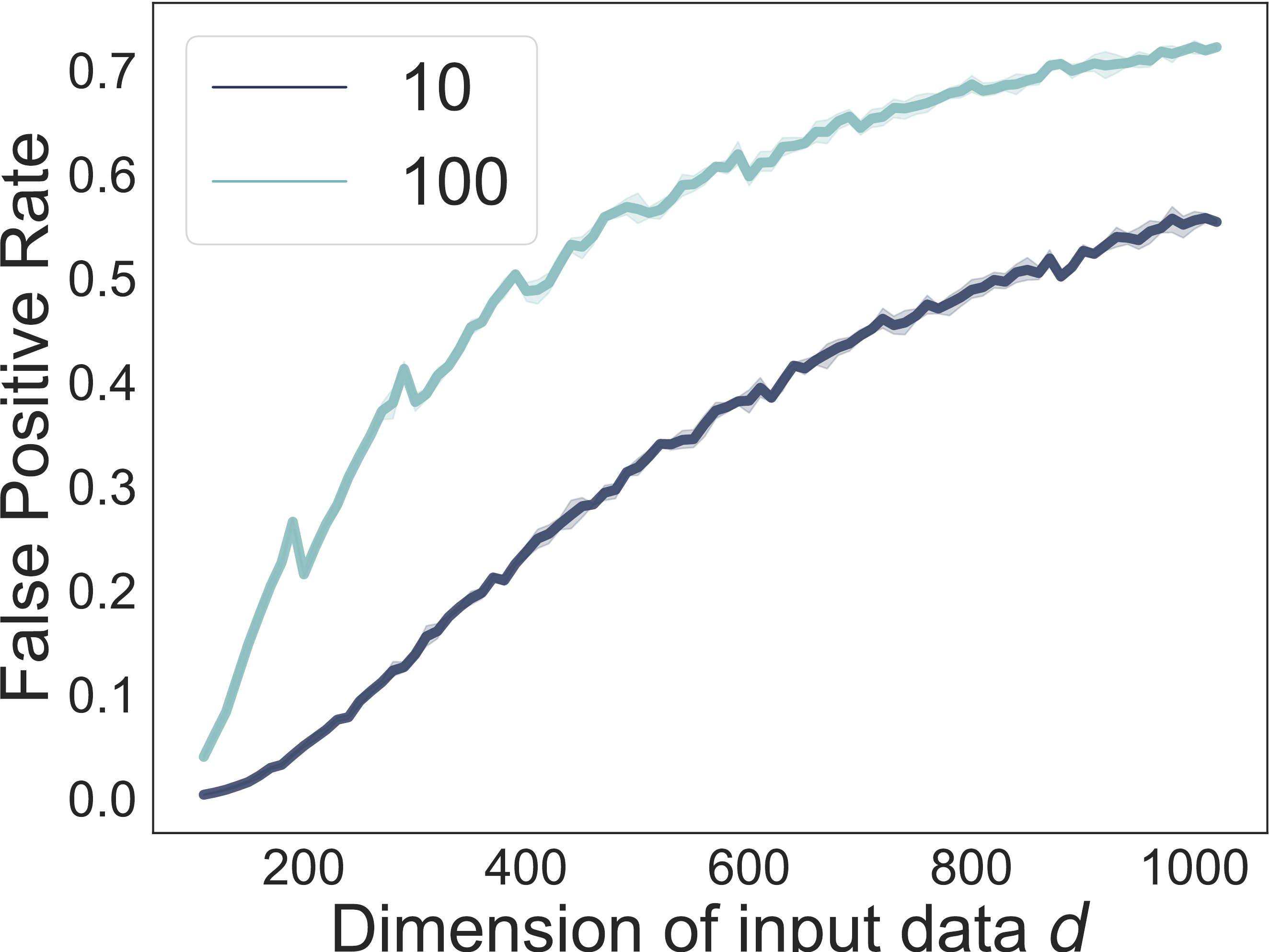}
\includegraphics[width=51mm]{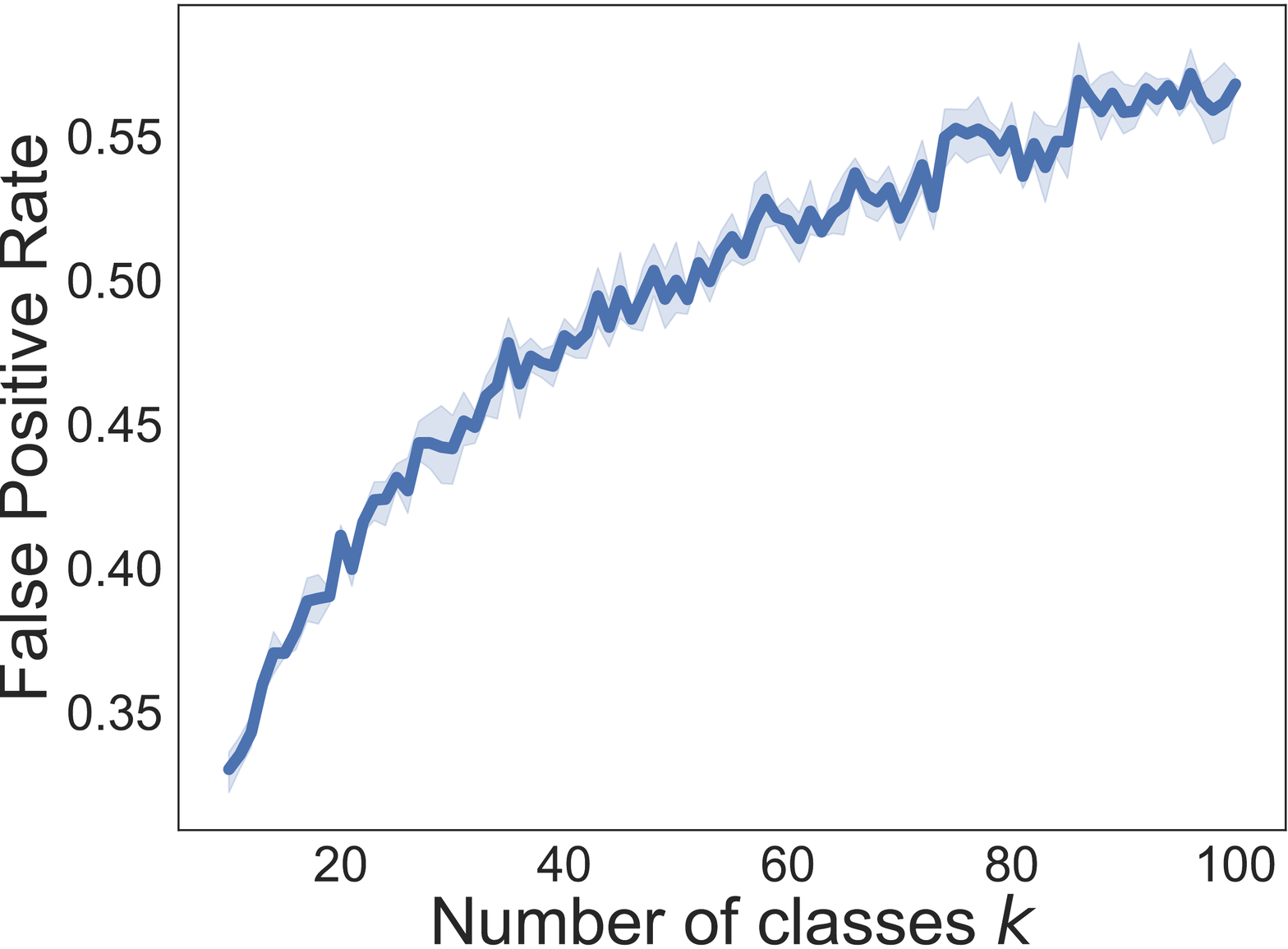}
\caption{\small 
distance between ID and OOD (\textbf{left}), the dimension of input data $d$ (\textbf{middle}), and number of classes (\textbf{right}). Each curve is averaged over 5 different runs (shade indicates the variance). A lower value on the y-axis is better. The covariance matrix $\Sigma$ is set to be identity in all of these simulations.
}
\label{fig:simulation} 
\end{center}
\end{figure}

\paragraph{Higher dimension exacerbates OOD uncertainty} Figure~\ref{fig:simulation} (middle) shows how the FPR changes as we increase the input dimension from $d=100$ to $d=1,000$ while keeping the distance fixed with $r=10$ and $\sigma=1$. As the dimension $d$ increases, the number of non-zero entries in each $\boldsymbol{\mu}_{i}$ increases accordingly (i.e. $\boldsymbol{\mu}_{i}$ becomes less sparse). Under the same feature dimension, we observe a higher FPR for $k=100$ than $k=10$, which corroborates the empirical observations on CIFAR-10 and CIFAR-100 (Section~\ref{sec:experiment-nn}). This suggests
that higher dimensions can be a key factor inducing the detrimental effect in OOD detection. 

\paragraph{Effect of the number of classes} 
Lastly, we investigate the performance of OOD uncertainty estimation by linearly increasing the number of classes $k$ from 10 to 100. We keep the magnitude parameter fixed with $r=10$ and dimension $d=512$ and $\sigma=1$. We see as the number of classes increases, the performance of our method decreases. 
We close this section by noting that we also provided formal mathematical justifications in the previous subsection.

\section{Related Work}
\label{sec:related}
Detecting unknowns has a long history in machine learning. We review works that are studied this problem in the context of deep neural networks. See  \citep{yang2021generalized} for a survey on generalized OOD detection (an umbrella term that includes closely related domains such as anomaly detection, novelty detection, open-set recognition, and OOD detection). 

\paragraph{Out-of-distribution detection for discriminative models}
In \citep{bendale2015towards}, the OpenMax  score is developed for OOD detection based on the extreme
value theory (EVT). Subsequent work by \citeauthor{hendrycks2016baseline} proposed a simple baseline using maximum softmax probability. The MSP score for OOD input is proven to be arbitrarily high for neural networks with ReLU activation~ \citep{hein2019relu}.  \citeauthor{liang2018enhancing} improved MSP by proposing the ODIN score, which amplifies the ID and OOD separability. It is shown that a sufficiently large temperature has a strong smoothing effect that transforms the softmax score back to the logit space---which more effectively distinguishes between ID vs. OOD. In \citep{lee2018simple}, a score is constructed based on the maximum Mahalanobis distance to the class means in the feature space of the pre-trained network. \citeauthor{liu2020energy}, proposed using the energy score, which can be derived directly from the logit output of the pre-trained network. In \citep{huang2021mos}, OOD detection is studied when the label space is large. It is shown that grouping the labels for in-distribution data can be effective in OOD detection for large semantic space. In \citep{ming2021impact}, the effect of spurious correlation is studied for OOD detection.  \citeauthor{huang2021importance} derived a scoring function termed GradNorm from the {gradient space}. GradNorm employs the vector norm of gradients, backpropagated from the KL divergence between the softmax output and a uniform probability distribution. In \citep{wang2021can}, the OOD detection is studied for multi-label classification where each data instance has multiple labels. In this work, we develop an analytical framework to analyze the performance of OOD scoring functions and show the superiority of GEM both theoretically and empirically. 

\vspace{-0.1cm}
\paragraph{Out-of-distribution detection via generative modeling}
There are several works that attempt modeling OOD data using generative modeling (e.g. GANs). \citeauthor{lee2017training} use GANs to generate data with low density for model regularization. \citeauthor{vernekar2019out} model in-distribution as a low dimensional submanifold of input space and uses auto-encoders to generate OOD samples outside of the in-distribution domain. \citeauthor{sricharan2018building} use GANs to generate OOD samples that the initial classifier is confident about and use those to create a more robust OOD detector.  Prior research also used generative modeling to estimate the density of the in-distribution data, and classify a sample as OOD if the estimated likelihood is low. 
However, it is shown in \citep{nalisnick2018deep} that deep generative models can produce a higher likelihood for OOD data. For example, it fails to distinguish CIFAR10 samples from SVHN. In \citep{ren2019likelihood} and \citep{serra2019input}, this problem is addressed by considering a likelihood ratio and taking the input complexity into account.

\paragraph{Out-of-distribution detection by model regularization}
Several works address the out-of-distribution detection problem during training-time regularization ~\citep{lee2017training, bevandic2018discriminative, geifman2019selectivenet,malinin2018predictive,mohseni2020self,jeong2020ood, chen2020informative}. In \citep{lee2017training}, a new term is added to the loss function of the neural net to force the out-of-distribution sample to have uniform prediction values across labels. A similar loss is followed by outlier exposure~\citep{hendrycks2018deep}. In \citep{liu2020energy}, a term is added to the loss function of the network to force out-distribution samples to have higher energy values after training. In \citep{chen2020informative}, an informative outlier mining procedure is proposed, which adaptively samples from auxiliary OOD data that is near the decision boundary between ID and OOD.
 
Such methods typically require having access to auxiliary unlabeled data. We focus on post hoc OOD detection methods, which have the advantages of being easy to use and general applicability. This is convenient for the adoption of OOD detection methods in real-world production environments, where the overhead cost of retraining or modifying the model can be prohibitive.

\paragraph{Uncertainty estimation in deep neural networks}
A Bayesian model is a statistical model that implements Bayes' rule to infer uncertainty within the model~\citep{jaynes1986bayesian}.
Recent works attempt several approximations of Bayesian inference including MC-dropout~\citep{mcdropout16icml} and deep ensembles~\citep{deepensemble00iwmcs,deepensemble17nips}. These methods address model uncertainty (\emph{i.e.}, epistemic) and are less competitive for OOD uncertainty estimation.  \citeauthor{kendall2017uncertainties} developed an extended framework to study aleatoric and epistemic uncertainty together. In \citep{van2020uncertainty} an uncertainty estimation method is developed using the RBF network. Dirichlet Prior Network (DPN) is also used for OOD detection with an uncertainty modeling of three different sources of uncertainty: model uncertainty, data uncertainty, and distributional uncertainty and form a line of works~\citep{dpn18nips,dpn19nips,dpn20nips}.

\section{Conclusion}
\label{sec:conclusion}
In this work, we develop an analytical framework that precisely characterizes and unifies the theoretical understanding of out-of-distribution detection. Our analytical framework motivates a novel OOD detection method for neural networks, \emph{GEM}, which demonstrates both theoretical and empirical superiority. We formally provide provable guarantees and comprehensive analysis of our method, underpinning how various properties of data distribution affect the performance of OOD detection. We hope our work can motivate future research on the theoretical understandings of OOD detection.

\section*{Acknowledgment}
The authors would like to thank Yifei Ming for the helpful discussion. This work is supported by the Wisconsin Alumni Research Foundation (WARF). YL is also supported by the Facebook Research Award, JP Morgan faculty award, and funding from Google Brain and American Family Insurance.

\bibliographystyle{plainnat}

\bibliography{AAAI.bib}
\appendix

 \begin{center}
\textbf{\large Supplemental Materials: Provable Guarantees for Understanding Out-of-distribution Detection
 }
\end{center}
 
The main goal of the appendix is to give detailed proofs for the results presented in the main paper, and also provide details for the experimental results. The outline of the appendix would be as follows. In Section \ref{app:background}, we recall the basic setting that we considered in the main paper. Next, we recall some notions and tools that we will need to present proofs. In Section \ref{app:energy-bound}, we give detailed proof for Theorem \ref{main_thm} (Theorem \ref{paper:main_thm} in the main paper). In Section \ref{app:other-proof}, we prove Proposition \ref{num-class-app} (Proposition \ref{num-class} in the main paper). 
Finally, in Section \ref{app:experiment-details}, we present experiment details on individual OOD data-sets.
\section{Background}
\label{app:background}
We start by briefly recalling the setting that we considered in the main paper. Let  $\mathcal{X}=\R^{d}$ be the input or feature space. Throughout the following, we work with a tied covariance matrix $\Sigma$. Let $P_{\mathcal{X}}^{\out}$ denotes the probability distribution on $\mathcal{X}$ with the following probability density,
\begin{align*}
	p_{\mathcal{X}}^{\out}(\mathbf{\*x})=\frac{1}{\sqrt{(2\pi)^{d}\abs{\Sigma}}}\exp(-\frac{1}{2}(\mathbf{x} -\boldsymbol{\mu}_{\out})^{\top}\Sigma^{-1}(\mathbf{x} -\boldsymbol{\mu}_{\out})).
\end{align*}
Next, let $P_{\mathcal{X}}^{\innn}$ be the probability distribution on $\mathcal{X}$ with the following probability density,
\begin{align*}
	p^{\innn}_{\mathcal{X}}(\mathbf{x})=\frac{1}{k}\sum_{i=1}^{k}p_{i}(\mathbf{x}),
\end{align*}
where,
\begin{align*}
	p_{i}(\*x)=\frac{1}{\sqrt{(2\pi)^{d}\abs{\Sigma}}}\exp(-\frac{1}{2}(\*x-\boldsymbol{\mu}_{i})^{\top}\Sigma^{-1}(\*x-\boldsymbol{\mu}_{i})).
\end{align*}
Next, recall that the re-scaled GEM function is defined by,
\begin{align*}
	ES(\*x)=\sum_{i=1}^{k}ES_{i}(\*x),
\end{align*}
where,
\begin{align*}
	ES_{i}(\*x)=\exp(-\frac{1}{2}(\*x-\boldsymbol{\mu}_{i})^{\top}\Sigma^{-1}(\*x-\boldsymbol{\mu}_{i})).
\end{align*}
Recall that Theorem \ref{main_thm} (Theorem \ref{paper:main_thm} in the main body) bounds,
\begin{align}
\label{ast}
D:=\mathbb{E}_{\*x \sim P^{\innn}_{\mathcal{X}}}	(ES(\*x))-\mathbb{E}_{\*x \sim P^{\out}_{\mathcal{X}}}	(ES(\*x)), \tag{$\ast$}
\end{align}
in terms of Mahalanobis distance between $\boldsymbol{\mu}_{\out}$ and $\boldsymbol{\mu}_{i}$. 
We start by setting notations and recalling definitions and tools that we need to prove Theorem \ref{main_thm} (Theorem \ref{paper:main_thm} in the main body).
Next, for the input or feature space $\mathcal{X}$, let $\B$ denotes the Borel $\sigma$-algebra on $\mathcal{X}$ and $P(\mathcal{X})$ denotes the set of all probability measures on $(\mathcal{X},\B)$. We recall the following definitions,
\begin{deff}[Total variation]
\label{TV}
Let $P_{1},P_{2} \in P(\mathcal{X})$. The total variation(TV) is defined by,
\begin{align*}
\delta(P_{1},P_{2})=\sup _{A\in \B}\abs{P_{1}(A)-P_{2}(A)}.
\end{align*}
\end{deff}
We use the following characterization of total variation (See for example \citep{TOT} Theorem 5.4), 
\begin{lem}
\label{total_char}
Let $P_{1},P_{2}\in P(\mathcal{X})$ and let $\mathcal{F}$ the unit ball in  $L^{\infty}(\mathcal{X})$,
\begin{align*}
\mathcal{F}:=\{f\in L^{\infty}(\mathcal{X})|\Norm{f}_{\infty}\le 1 \},
\end{align*}
then we have the following characterization for the total variation distance,
\begin{align*}
\delta(P_{1},P_{2})=\sup_{f \in \mathcal{F}	}\abs{\mathbb{E}_{\*x \sim P_{1}}f(\*x)-\mathbb{E}_{\*x \sim P_{2}}f(\*x)}.
\end{align*}
\end{lem}
Next, recall definition of Kullback–Leibler(KL) divergence,
\begin{deff}[Kullback–Leibler divergence]
\label{KL}
Let $P_{1},P_{2}\in P(\mathcal{X})$ be two probability measures with density functions $p_{1}$ and $p_{2}$ respectively. The Kullback–Leibler(KL) divergence is defined by,
\begin{align*} 
      KL(P_{1}||P_{2}):=\int_{\mathcal{X}}\ln(\frac{p_{1}(\*x)}{p_{2}(\*x)})p_{1}(\*x)d\*x,
\end{align*}
whenever the above integral is defined. 
\end{deff}
Next, recall the following  standard lemma that computes KL divergence between {\it multivariate normal (MVN)} distributions,
\begin{lem}
\label{KL_Gauss}
Let $P_{1}\sim \mathcal{N}(\boldsymbol{\mu}_{1},\Sigma)$ and $P_{2}\sim \mathcal{N}(\boldsymbol{\mu}_{2},\Sigma)$ then we have the following,
\begin{align*}
KL(P_{1}||P_{2})=\frac{1}{2}((\boldsymbol{\mu}_{1}-\boldsymbol{\mu}_{2})^{\top}\Sigma^{-1}(\boldsymbol{\mu}_{1}-\boldsymbol{\mu}_{2}))=\frac{1}{2}d^{2}_{M}(\boldsymbol{\mu}_{1},\boldsymbol{\mu}_{2}).
\end{align*}
\end{lem}
Next, we recall the following inequality that bounds the total variation by KL-divergence. (See for example \citep{AlexT}, Lemma 2.5 and Lemma 2.6])
\begin{lem}[Pinsker inequality]
\label{Pin}
Let $P_{1},P_{2}\in P(\mathcal{X})$ then we have the following,
\begin{align*}
\delta(P_{1},P_{2})\le 	\sqrt{\frac{1}{2}KL(P_{1}||P_{2})},
\end{align*}
the following version, which is useful when KL is large, is also true,
\begin{align*}
\delta(P_{1},P_{2})\le 1-\frac{1}{2}\exp(-KL(P_{1}||P_{2})).	
\end{align*}
\end{lem}
\section{Proof of the Energy Bounds}
\label{app:energy-bound}
We start by proving the lower bound for $D$. We first prove the following lemma,
\begin{lem}
\label{Lower}
For $1\le i\le k$, let $\alpha_{i}:=\frac{1}{2}d_{M}(\boldsymbol{\mu}_{i},\boldsymbol{\mu}_{\out})$ then we have the following estimate,
\begin{align*}
\mathbb{E}_{\*x \sim P^{\out}_{\mathcal{X}}}	(ES_{i}(\*x))\le 	\big{(}1-P_{\mathcal{X}}^{\out}(B_{\alpha_{i}}(\boldsymbol{\mu}_{\out}))\big{)}+\exp(-\frac{\alpha_{i}^{2}}{2})P_{\mathcal{X}}^{\out}(B_{\alpha_{i}}(\boldsymbol{\mu}_{\out})).	\end{align*}
\end{lem}
\begin{proof}
We have,
\begin{align*}
	&\mathbb{E}_{\*x \sim P^{\out}_{\mathcal{X}}}	(ES_{i}(\*x))=\int_{\R^{d}}ES_{i}(\*x)p_{\mathcal{X}}^{\out}(\*x)d\*x=\int_{B_{\alpha_{i}}(\boldsymbol{\mu}_{\out})}ES_{i}(\*x)p_{\mathcal{X}}^{\out}(\*x)d\*x+\int_{B_{\alpha_{i}}^{c}(\boldsymbol{\mu}_{\out})}ES_{i}(\*x)p_{\mathcal{X}}^{\out}(\*x)d\*x.
\end{align*}
Next, for $\*x\in B_{\alpha_{i}}(\boldsymbol{\mu}_{\out})$, by triangle inequality and definition we have,
\begin{align*}
&d_{M}(\boldsymbol{\mu}_{\out},\*x)+d_{M}(\*x,\boldsymbol{\mu}_{i})\ge d_{M}(\boldsymbol{\mu}_{\out},\boldsymbol{\mu}_{i})\implies d_{M}(\*x,\boldsymbol{\mu}_{i})	\ge \alpha_{i}=\frac{1}{2}d_{M}(\boldsymbol{\mu}_{i},\boldsymbol{\mu}_{\out}).
\end{align*}
So for the first term, we have,
\begin{align*}
	&\int_{B_{\alpha_{i}}(\boldsymbol{\mu}_{\out})}ES_{i}(\*x)p_{\mathcal{X}}^{\out}(\*x)d\*x=\int_{B_{\alpha_{i}}(\boldsymbol{\mu}_{\out})}\exp(-\frac{1}{2}d_{M}^{2}(\*x,\boldsymbol{\mu}_{i}))p^{\out}(\*x)d\*x \le \exp(-\frac{1}{2}\alpha_{i}^{2})P_{\mathcal{X}}^{\out}(B_{\alpha_{i}}(\boldsymbol{\mu}_{\out})).
\end{align*}
For the second term, since $ES_{i}(\*x)\le 1$, we have,
\begin{align*}
\int_{B_{\alpha_{i}}^{c}(\boldsymbol{\mu}_{\out})}ES_{i}(\*x)p_{\mathcal{X}}^{\out}(\*x)d\*x\le 	\big{(}1-P_{\mathcal{X}}^{\out}(B_{\alpha_{i}}(\boldsymbol{\mu}_{\out}))\big{)}.\end{align*}
Putting all together we have,
\begin{align*}
	&\mathbb{E}_{\*x \sim P^{\out}_{\mathcal{X}}}ES_{i}(\*x)\le \big{(}1-P_{\mathcal{X}}^{\out}(B_{\alpha_{i}}(\boldsymbol{\mu}_{\out}))\big{)}+\exp(-\frac{1}{2}\alpha_{i}^{2})P_{\mathcal{X}}^{\out}(B_{\alpha_{i}}(\boldsymbol{\mu}_{\out})),
\end{align*}
and we are done.
\end{proof}
\begin{prop}
\label{LowerProp}
We have the following estimate,
\begin{align*}
\mathbb{E}_{\*x \sim P^{\out}_{\mathcal{X}}}	(ES(\*x))\le 	\sum_{i=1}^{k}\big{(}1-P_{\mathcal{X}}^{\out}(B_{\alpha_{i}}(\boldsymbol{\mu}_{\out}))\big{)}+\sum_{i=1}^{k}\exp(-\frac{\alpha_{i}^{2}}{2})P_{\mathcal{X}}^{\out}(B_{\alpha_{i}}(\boldsymbol{\mu}_{\out})),	\end{align*}
where, $\alpha_{i}:=\frac{1}{2}d_{M}(\boldsymbol{\mu}_{i},\boldsymbol{\mu}_{\out})$.
\end{prop}
\begin{proof}
The proof is a direct application of Lemma \ref{Lower}.	
\end{proof}
\begin{cor}
\label{clean-lower}
Notice that since $P_{\mathcal{X}}^{\out}(B_{\alpha_{i}}(\boldsymbol{\mu}_{\out}))\le 1$ we also have the following bound from Proposition \ref{LowerProp}.
\begin{align*}
\mathbb{E}_{\*x \sim P^{\out}_{\mathcal{X}}}	(ES(\*x))\le 	\sum_{i=1}^{k}\big{(}1-P_{\mathcal{X}}^{\out}(B_{\alpha_{i}}(\boldsymbol{\mu}_{\out}))\big{)}+\sum_{i=1}^{k}\exp(-\frac{\alpha_{i}^{2}}{2}),	\end{align*}
where, $\alpha_{i}:=\frac{1}{2}d_{M}(\boldsymbol{\mu}_{i},\boldsymbol{\mu}_{\out})$.
\end{cor}
Next, we state and prove the following proposition that gives an upper bound for $D$ in terms of  Mahalanobis distance between in-distribution means and out-distribution mean,
\begin{prop}
\label{energy-bound}
For $1\le i\le k$, set $\alpha_{i}:=\frac{1}{2}d_{M}(\boldsymbol{\mu}_{i},\boldsymbol{\mu}_{\out})$. We have the following,	
\begin{align*}
	\mathbb{E}_{\*x \sim P^{\innn}_{\mathcal{X}}}	(ES(\*x))-\mathbb{E}_{\*x \sim P^{\out}_{\mathcal{X}}}	(ES(\*x))\le \sum_{i=1}^{k}\alpha_{i}.
\end{align*}
\end{prop}
\begin{proof}
First, notice that, for $1\le i\le k$,
\begin{align*}
  ES_{i}(\*x)\in [0,1]\implies ES(\*x)\in [0,k].  
\end{align*}
Therefore, by Lemma \ref{total_char}, we have,
\begin{align*}
\mathbb{E}_{\*x \sim P^{\innn}_{\mathcal{X}}}	(ES(\*x))-\mathbb{E}_{\*x \sim P^{\out}_{\mathcal{X}}}	(ES(\*x))\le k\cdot \delta (	P^{\innn}_{\mathcal{X}},P^{\out}_{\mathcal{X}}).
\end{align*}
Next, recall,
\begin{align*}
	p^{\innn}_{\mathcal{X}}(\*x)=\frac{1}{k}\sum_{i=1}^{k}p_{i}(\*x),
\end{align*}
therefore, let $P_{i}$ denotes the probability distribution correspond to $p_{i}$ and by triangle inequality and the definition of total variation we obtain,
\begin{align*}
	\delta(P^{\innn}_{\mathcal{X}}, P^{\out}_{\mathcal{X}})=\delta(\frac{1}{k}\sum_{i=1}^{k}P_{i}, P^{\out}_{\mathcal{X}})&=\sup_{A\in \B}\abs{\frac{1}{k}\sum_{i=1}^{k}P_{i}(A)-P^{\out}_{\mathcal{X}}(A)}\\
	&\le \frac{1}{k}\sum_{i=1}^{k}\sup_{A \in \B}\abs{P_{i}(A)-P^{\out}_{\mathcal{X}}(A)}\\
	&= \frac{1}{k}\sum_{i=1}^{k}\delta(P_{i},P^{\out}_{\mathcal{X}}).\end{align*}
Finally, by Pinsker's inequality and Lemma \ref{KL_Gauss} we have,
\begin{align*}
	\delta(P_{i}, P^{\out}_{\mathcal{X}})\le \sqrt{\frac{1}{2}KL(P_{i}||P^{\out}_{\mathcal{X}})}=\frac{1}{2}\sqrt{(\boldsymbol{\mu}_{i}-\boldsymbol{\mu}_{out})^{\top}\Sigma^{-1}(\boldsymbol{\mu}_{i}-\boldsymbol{\mu}_{out})}=\frac{1}{2}d_{M}(\boldsymbol{\mu}_{i},\boldsymbol{\mu}_{\out}).
\end{align*}
Putting all together we obtain,
\begin{align*}
\mathbb{E}_{\*x \sim P^{\innn}_{\mathcal{X}}}	(ES(\*x))-\mathbb{E}_{\*x \sim P^{\out}_{\mathcal{X}}}	(ES(\*x))\le \frac{1}{2}\sum_{i=1}^{k}\sqrt{(\boldsymbol{\mu}_{i}-\boldsymbol{\mu}_{out})^{\top}\Sigma^{-1}(\boldsymbol{\mu}_{i}-\boldsymbol{\mu}_{out})}=\frac{1}{2}\sum_{i=1}^{k}d_{M}(\boldsymbol{\mu}_{i},\boldsymbol{\mu}_{\out}),
\end{align*}
and the proof is complete. 
\end{proof}
 \begin{rmk}
Proposition \ref{energy-bound} will be useful when $d_{M}(\boldsymbol{\mu}_{i},\boldsymbol{\mu}_{out})$ are small. Notice that when $d_{M}(\boldsymbol{\mu}_{out},\boldsymbol{\mu}_{i})$ is large then the right hand side of Proposition \ref{energy-bound} is loose. However, in such a case we can instead use the second version of Pinsker's inequality in Lemma \ref{Pin} and obtain the following, 	
\begin{align*}
\mathbb{E}_{\*x \sim P^{\innn}_{\mathcal{X}}}	(ES(\*x))-\mathbb{E}_{\*x \sim P^{\out}_{\mathcal{X}}}	(ES(\*x))\le k-\frac{1}{2}\sum_{i=1}^{k}\exp(-2\alpha_{i}^{2}).
\end{align*}
The proof would be similar to the proof of Proposition \ref{energy-bound}. 
\end{rmk}
Finally we combine all we proved above in the following theorem ,
\begin{thm}
\label{main_thm}
We have the following bounds,	
\begin{itemize}
    \item 
    
	$\mathbb{E}_{\*x \sim P^{\out}_{\mathcal{X}}}	(ES(\*x))\le \sum_{i=1}^{k}\Big{(}\big{(}1-P_{\mathcal{X}}^{\out}(B_{\alpha_{i}}(\boldsymbol{\mu}_{\out}))\big{)}+\exp(-\frac{1}{2}\alpha_{i}^{2})\Big{)}$,
	\item 
	$\mathbb{E}_{\*x \sim P^{\innn}_{\mathcal{X}}}	(ES(\*x))-\mathbb{E}_{\*x \sim P^{\out}_{\mathcal{X}}}	(ES(\*x))\le \sum_{i=1}^{k}\alpha_{i}.$
\end{itemize}
where , $\alpha_{i}:=\frac{1}{2}d_{M}(\boldsymbol{\mu}_{i},\boldsymbol{\mu}_{\out})$ for $1\le i\le k$.
\end{thm}
\begin{proof}
The first bound follows from Corollary \ref{clean-lower} and the second bound follows from Proposition \ref{energy-bound}. 
\end{proof}

\begin{cor}
Let $\boldsymbol{\mu}_{\out}=\boldsymbol{0}\in \R^{d}$ and $\boldsymbol{\mu}_{i}$ with the configuration that we considered in Section \ref{sec:bound} of the main body and let $\Sigma=\mathbf{I}_{d}$. It follows that $\alpha=d_{M}(\boldsymbol{\mu}_{out},\boldsymbol{\mu}_{i})$ for all $1\le i\le k$. We have the following from the first bound in the Theorem \ref{main_thm},
\begin{align*}
	& \mathbb{E}_{\*x \sim P^{\out}_{\mathcal{X}}}	(ES(\*x))\le k\Big{(}\big{(}1-P_{\mathcal{X}}^{\out}(B_{\alpha}(\boldsymbol{\mu}_{\out}))\big{)}+\exp(-\frac{1}{2}\alpha^{2}))\Big{)}.
\end{align*}
Now as $\alpha \to \infty$ the right hand side in the above approaches to $0$. This shows that the performance of GEM becomes better as $\alpha \to \infty$. On the other hand, we have the following from second part of Theorem \ref{main_thm},
\begin{align*}
	 \mathbb{E}_{x \sim P^{\innn}_{\mathcal{X}}}	(ES(\*x))-\mathbb{E}_{\*x \sim P^{\out}_{\mathcal{X}}}	(ES(\*x))\le k\alpha,
\end{align*}
and it follows that as $\alpha \to 0$ the energy difference between in-distribution and out-distribution data converges to $0$. In other words, the performance of GEM  decreases as $\alpha$ approaches to $0$. 
 \end{cor}
\section{Other Proofs}
\label{app:other-proof}
\paragraph{Performance of our method with respect to the number of classes}
In this section, we work in the setting same as Section \ref{sec:bound} from the main paper. Recall that,  $\boldsymbol{\mu}_{\out}=\boldsymbol{0}\in \R^{d}$ and $\Sigma=\mathbf{I}_{d}$ and $\boldsymbol{\mu}_{i}$ are such that for $i,j \in \{1,...,k\}$ and $i\neq j$,
\begin{align*}
    &\Norm{\boldsymbol{\mu}_{i}-\boldsymbol{\mu}_{j}}=\sqrt{2}\alpha ,\\
    &\Norm{\boldsymbol{\mu}_{i}-\boldsymbol{\mu}_{\out}}=\alpha.
\end{align*}
We compute $D$, defined by \ref{ast}, in terms of $k$ to see how it is related to the number of classes. First, we need the following definition,
\begin{deff}
Let $\boldsymbol{\mu} ,\boldsymbol{\nu} \in \R^{d}$ with $\Norm{\boldsymbol{\mu}-\boldsymbol{\nu}}=\gamma$. Let $P \sim \mathcal{N}(\boldsymbol{\mu},\mathbf{I}_{d})$. Define,
\begin{align*}
    A_{\gamma}:=\mathbb{E}_{\*x \sim P}	(\exp(-\frac{1}{2}\Norm{\*x-\boldsymbol{\nu}}^{2})).
\end{align*}
\end{deff}
\begin{rmk}
\label{fixe-gaussian-distance}
Notice that $A_{\gamma}$ only depends on the distance between $\boldsymbol{\mu}$ and $\boldsymbol{\nu}$ (i.e. $\gamma$). Also it is easy to see that $A_{\gamma}$ decreases as $\gamma$ increases.
\end{rmk}

Next, we state and prove the following proposition,
\begin{prop}
\label{num-class-app}
We have the following,
\begin{align*}
 D=A_{0}-A_{\alpha}+(k-1)(A_{\sqrt{2}\alpha}-A_{\alpha})
\end{align*}
\end{prop}
\begin{proof}
First, recall from \ref{ast} that,
\begin{align*}
    D=\mathbb{E}_{\*x \sim P^{\innn}_{\mathcal{X}}}	(ES(\*x))-\mathbb{E}_{\*x \sim P^{\out}_{\mathcal{X}}}	(ES(\*x)).
\end{align*}
Next, for the first term we have,
\begin{equation}
\begin{aligned}
\label{I}
 \mathbb{E}_{\*x \sim P^{\innn}_{\mathcal{X}}}	(ES(\*x))=\sum_{i=1}^{k} \mathbb{E}_{\*x \sim P^{\innn}_{\mathcal{X}}}  (ES_{i}(\*x))=\frac{1}{k}\sum_{i=1}^{k} \sum_{j=1}^{k}\mathbb{E}_{\*x \sim P_{j}}  (ES_{i}(\*x))&=\frac{1}{k}\big{(}kA_{0}+k(k-1)A_{\sqrt{2}\alpha}\big{)}\\
 &=A_{0}+(k-1)A_{\sqrt{2}\alpha}.
\end{aligned}\tag{I}
\end{equation}
For the second term we have,
\begin{align}
\label{II}
     \mathbb{E}_{\*x \sim P^{\out}_{\mathcal{X}}}	(ES(\*x))=\sum_{i=1}^{k} \mathbb{E}_{\*x \sim P^{\out}_{\mathcal{X}}}  (ES_{i}(\*x))=kA_{\alpha}.\tag{II}
\end{align}
The result follows by combining \ref{I} and \ref{II}.
\end{proof}
\begin{cor}
Since $\sqrt{2}\alpha>\alpha$ it follows from Remark \ref{fixe-gaussian-distance} that $A_{\sqrt{2}\alpha} <A_{\alpha}$. This means that the last term term in the following is negative,
\begin{align*}
   D=A_{0}-A_{\alpha}+(k-1)(A_{\sqrt{2}\alpha}-A_{\alpha}), 
\end{align*}
In other words, as $k$ increases $D$ becomes smaller which indicates that the performance decreases. 
\end{cor}
\section{Experiment Details}
\label{app:experiment-details}
All experiments ran with PyTorch and NVIDIA GeForce RTX 2080Ti. The pre-trained models are downloaded from the codebase: \url{https://github.com/wetliu/energy_ood}.
\begin{table*}
               \centering
               \scalebox{0.9}{
               \begin{tabular}{ll|ccc}
                        \toprule
          \multirow{4}{0.06\linewidth}{\textbf{Dataset $\mathcal{D}_{\text{out}}^{\text{test}}$}}  &  &\textbf{FPR95}  &  \textbf{AUROC}  & \textbf{AUPR}   \\
                 & &  & $\textbf{}$  & \\
                        & & $\downarrow$ & $\uparrow$ & $\uparrow$ \\
                        \hline
                        \multirow{5}{0.09\linewidth}{{\textbf{Texture}}}
                        & Softmax score \citep{hendrycks2016baseline} & 59.28 & 88.50 & 97.16 \\
                        & Energy score ~\citep{liu2020energy} & 52.79 & 85.22 & 95.41 \\
                        & ODIN \citep{liang2018enhancing} & 49.12 & 84.97 & 95.28 \\
                        & Mahalanobis \citep{lee2018simple} & 15.14 & 97.32 & 99.41 \\
                        & GEM  (ours)  & 15.06 & 97.33 & 99.41\\
                        \hline
                        \multirow{5}{0.09\linewidth}{{\textbf{SVHN}}}
                        & Softmax score \citep{hendrycks2016baseline} & 48.49 & 91.89 & 98.27 \\
                        & Energy score ~\citep{liu2020energy} & 35.59 & 90.96 & 97.64 \\
                        & ODIN \citep{liang2018enhancing} & 33.55 & 91.96 & 98.00 \\
                        & Mahalanobis \citep{lee2018simple} & 12.86 & 97.59 & 99.47 \\
                        & GEM  (ours)  & 13.42 & 97.59 & 99.47\\
                        \hline
                        \multirow{5}{0.09\linewidth}{{\textbf{ Places365}}}
                        & Softmax score \citep{hendrycks2016baseline} & 59.48 & 88.20 & 97.10 \\
                        & Energy score ~\citep{liu2020energy} & 40.14 & 89.89 & 97.30 \\
                        & ODIN \citep{liang2018enhancing} & 57.40 & 84.49 & 95.82 \\
                        & Mahalanobis \citep{lee2018simple} & 68.42 & 84.41 & 96.08  \\

                        & GEM  (ours)  & 68.03 & 84.44 & 96.11\\
                        \hline
                        \multirow{5}{0.09\linewidth}{{\textbf{ LSUN-C}}}
                        & Softmax score \citep{hendrycks2016baseline} & 30.80 & 95.65 & 99.13 \\
                        & Energy score ~\citep{liu2020energy} & 8.26 & 98.35 & 99.66 \\
                        & ODIN \citep{liang2018enhancing} & 15.52 & 97.04 & 99.33 \\
                        &  Mahalanobis \citep{lee2018simple} & 39.47 & 94.09 & 98.80 \\
                        & GEM  (ours)  & 39.46 & 94.13 & 98.81\\
                        \hline
                        \multirow{5}{0.09\linewidth}{{\textbf{ LSUN Resize}}}
                        & Softmax score \citep{hendrycks2016baseline} & 52.15 & 91.37 & 98.12 \\
                        & Energy score ~\citep{liu2020energy} & 27.58 & 94.24 & 98.67 \\
                        & ODIN \citep{liang2018enhancing} & 26.62 & 94.57 & 98.77 \\
                        & Mahalanobis \citep{lee2018simple} & 42.07 & 93.29 & 98.61 \\
                        & GEM  (ours)  & 42.89 & 93.27 & 98.61\\
                        \hline
                        \multirow{5}{0.09\linewidth}{{\textbf{iSUN}}}
                        & Softmax score \citep{hendrycks2016baseline} & 56.03 & 89.83 & 97.74 \\
                        & Energy score ~\citep{liu2020energy} & 33.68 & 92.62 & 98.27 \\
                        & ODIN \citep{liang2018enhancing} & 32.05 & 93.50 & 98.54 \\
                        & Mahalanobis \citep{lee2018simple} & 43.80 & 92.75 & 98.46 \\
                        & GEM  (ours)  & 44.41 & 92.60 & 98.42\\
                        \bottomrule
                \end{tabular}}
        \vspace{-0.2cm}
        \caption[]{\small OOD Detection performance of CIFAR-10 as in-distribution for each OOD test dataset. The reported results for benchmarks other than GEM are courtesy of \citep{liu2020energy}
        }
        \label{tab:cifar10-results}
\end{table*}

\begin{table*}
               \centering
               \scalebox{0.9}{
               \begin{tabular}{ll|ccc}
                        \toprule  \multirow{4}{0.06\linewidth}{\textbf{Dataset $\mathcal{D}_{\text{out}}^{\text{test}}$}} & &\textbf{FPR95}  &  \textbf{AUROC}  & \textbf{AUPR}   \\
                 & &  & $\textbf{}$  & \\
                        & & $\downarrow$ & $\uparrow$ & $\uparrow$ \\
                        \hline
                        \multirow{5}{0.09\linewidth}{{\textbf{ Texture}}}
                        & Softmax score \citep{hendrycks2016baseline} & 83.29 & 73.34 & 92.89  \\
                        & Energy score ~\citep{liu2020energy} & 79.41 & 76.28 & 93.63  \\
                        & ODIN \citep{liang2018enhancing} & 79.27 & 73.45 & 92.75\\
                        & Mahalanobis \citep{lee2018simple} & 39.66 & 90.79 & 97.82 \\
                        & GEM (ours)  & 39.77 & 90.93 & 97.86\\
                        \hline
                        \multirow{5}{0.09\linewidth}{{\textbf{ SVHN}}}
                        & Softmax score \citep{hendrycks2016baseline} & 84.59 & 71.44 & 92.93 \\
                        & Energy score ~\citep{liu2020energy} & 85.82 & 73.99 & 93.65 \\
                        & ODIN \citep{liang2018enhancing} & 84.66 & 67.26 & 91.38 \\
                        & Mahalanobis \citep{lee2018simple} & 43.93 & 90.52 & 97.84 \\
                        & GEM  (ours)  & 43.87 & 90.49 & 97.83\\
                        \hline
                        \multirow{5}{0.09\linewidth}{{\textbf{ Places365}}}
                        & Softmax score \citep{hendrycks2016baseline} & 82.84 & 73.78 & 93.29\\
                        & Energy score ~\citep{liu2020energy} & 80.56 & 75.44 & 93.45 \\
                        & ODIN \citep{liang2018enhancing} & 87.88 & 71.63 & 92.56 \\
                        &  Mahalanobis \citep{lee2018simple} & 90.12 & 68.44 & 91.21 \\
                        & GEM  (ours)  & 90.43 & 68.15 & 91.08\\
                        \hline
                        \multirow{5}{0.09\linewidth}{{\textbf{ LSUN-C}}}
                        & Softmax score \citep{hendrycks2016baseline} & 66.54 & 83.79 & 96.35 \\
                        & Energy score ~\citep{liu2020energy} & 35.32 & 93.53 & 98.62 \\
                        & ODIN \citep{liang2018enhancing} & 55.55 & 87.73 & 97.22 \\
                        & Mahalanobis \citep{lee2018simple} & 98.47 & 60.23 & 90.32 \\
                        & GEM  (ours)  & 98.31 & 60.03 & 90.22\\
                        \hline
                        \multirow{5}{0.09\linewidth}{{\textbf{ LSUN Resize}}}
                        & Softmax score \citep{hendrycks2016baseline} & 82.42 & 75.38 & 94.06\\
                        & Energy score ~\citep{liu2020energy} & 79.47 & 79.23 & 94.96 \\
                        & ODIN \citep{liang2018enhancing}& 71.96 & 81.82 & 95.65\\
                        &  Mahalanobis \citep{lee2018simple} & 33.35 & 93.74 & 98.64 \\
                        & GEM  (ours)  & 33.13 & 93.82 & 98.66\\
                        \hline
                        \multirow{5}{0.09\linewidth}{{\textbf{iSUN}}}
                        & Softmax score \citep{hendrycks2016baseline} & 82.80 & 75.46 & 94.06 \\
                        & Energy score ~\citep{liu2020energy} & 81.04 & 78.91 & 94.91 \\
                        & ODIN \citep{liang2018enhancing} & 68.51 & 82.69 & 95.80 \\
                        &  Mahalanobis \citep{lee2018simple} & 36.52 & 92.48 & 98.28 \\
                        & GEM  (ours)  & 36.68 & 92.52 & 98.30\\
                        \bottomrule
                \end{tabular}}
        \vspace{-0.2cm}
        \caption[]{\small OOD Detection performance of CIFAR-100 as in-distribution for each OOD test dataset. The reported results for benchmarks other than GEM are courtesy of \citep{liu2020energy}} 
        
        \label{tab:cifar100-results}
\end{table*}
\newpage
\text{  }

\end{document}